\pdfoutput=1

\documentclass[11pt]{article}
\usepackage[T1]{fontenc}
\usepackage[utf8]{inputenc}
\usepackage{mathtools}
\usepackage{makecell}
\usepackage{thm-restate}

\newcommand{\norms}[1]{\left\|#1\right\|^2_2}
\title{Differentially Private Assouad, Fano, and Le Cam}

\author{Jayadev Acharya\thanks{Research supported by NSF 1815893, NSF 1657471, and NSF 1846300 (CAREER).}\\
Cornell University\\
\tt{acharya@cornell.edu}
\and
Ziteng Sun$^*$\\
Cornell University\\
\tt{zs335@cornell.edu}
\and
Huanyu Zhang$^*$\\
Cornell University\\
\tt{hz388@cornell.edu}
}

\usepackage{lmodern}
\usepackage{xspace}
\usepackage[protrusion=true,expansion=true]{microtype}
\usepackage{fullpage}
\usepackage{amsfonts,amsmath,amssymb,amsthm, pbox}

\usepackage[usenames,dvipsnames,table]{xcolor}
\usepackage[colorlinks,citecolor=blue, bookmarks=true]{hyperref}

\usepackage{multirow}
\usepackage{dsfont} 

\usepackage[shortlabels]{enumitem}
\setitemize{noitemsep,topsep=0pt,parsep=0pt,partopsep=0pt}
\setenumerate{noitemsep,topsep=0pt,parsep=0pt,partopsep=0pt}

\makeatletter
\@ifundefined{theorem}{
  \theoremstyle{plain}
  \newtheorem{theorem}{Theorem}

  \newtheorem{lemma}[theorem]{Lemma}

   \theoremstyle{definition}
  
  \theoremstyle{remark}

}{}
\makeatother

\newcommand{\ignore}[1]{}

\newcommand{\II}{\mathbb{I}} 
\newcommand{\EE}{\mathbb{E}}

\newcommand{\RR}{\mathbb{R}}

\newcommand{\expectation}[1]{\EE\left[#1\right]}

\def \cA     {{\cal A}}
\def \cB     {{\cal B}}
\def \cC     {{\cal C}}

\def \cE     {{\cal E}}

\def \cG     {{\cal G}}
\def \cH     {{\cal H}}

\def \cN     {{\cal N}}

\def \cP     {{\cal P}}
\def \cQ     {{\cal Q}}

\def \cS     {{\cal S}}

\def \cV     {{\cal V}}

\def \cX     {{\cal X}}

\newcommand{\ie}{\textit{i.e.,}\xspace}  

\newcommand{\absv}[1]{\left|#1\right|}
\newcommand{\norm}[1]{\left\|#1\right\|_2}

\def \Paren#1{{\left({#1}\right)}}

\def \Brack#1{{\left[{#1}\right]}}

\newcommand{\probof}[1]{\Pr\Paren{#1}}

\def\ignore#1{}

\newcommand{\bi}{\begin{itemize}}
\newcommand{\ei}{\end{itemize}}

\def\orpro{\mathop{\mathchoice
   {\vee\kern-.49em\raise.7ex\hbox{$\cdot$}\kern.4em}
   {\vee\kern-.45em\raise.63ex\hbox{$\cdot$}\kern.2em}
   {\vee\kern-.4em\raise.3ex\hbox{$\cdot$}\kern.1em}
   {\vee\kern-.35em\raise2.2ex\hbox{$\cdot$}\kern.1em}}\limits}

\def\andpro{\mathop{\mathchoice
 {\wedge\kern-.46em\lower.69ex\hbox{$\cdot$}\kern.3em}
 {\wedge\kern-.46em\lower.58ex\hbox{$\cdot$}\kern.25em}
 {\wedge\kern-.38em\lower.5ex\hbox{$\cdot$}\kern.1em}
 {\wedge\kern-.3em\lower.5ex\hbox{$\cdot$}\kern.1em}}\limits}

\def\simge{\mathrel{%
   \rlap{\raise 0.511ex \hbox{$>$}}{\lower 0.511ex \hbox{$\sim$}}}}

\def\simle{\mathrel{
   \rlap{\raise 0.511ex \hbox{$<$}}{\lower 0.511ex \hbox{$\sim$}}}}

 \newcommand{\newzs}[1]{{\color{black}#1}}
 
 \newcommand{\newhz}[1]{{\color{black}#1}}

  \newcommand{\newhzn}[1]{{\color{black}#1}}

\newcommand{\ab}{k}
\newcommand{\dims}{d}
\newcommand{\ns}{n}
\newcommand{\ber}{\cB}
\newcommand{\dP}{p}
\newcommand{\dQ}{q}
\newcommand{\mP}{p}

\newcommand{\p}{p}
\newcommand{\q}{q}

\newcommand{\expectationsub}[2]{\EE_{#1}\left[#2\right]}
\newcommand{\probofsub}[2]{\Pr\nolimits_{#1}\Paren{#2}}

\newcommand{\absvn}[1]{\left\|#1\right\|_1}

\newcommand{\Xon}{X^\ns}

\newcommand{\eps}{\varepsilon}
\newcommand{\he}{\hat{E}}

\newcommand{\ham}[2]{d_{\rm Ham}(#1,#2)}
\newcommand{\alg}{\hat \theta}

\newcommand{\dist}{\alpha}

\newcommand{\dtv}[2]{d_{TV}(#1, #2)}
\newcommand{\dkl}[2]{d_{KL}(#1, #2)}

\usepackage{thmtools}
\usepackage{thm-restate}

\begin{document}


\maketitle 

\begin{abstract}
Le Cam's method, Fano's inequality, and Assouad's lemma are three widely used techniques to prove lower bounds for statistical estimation tasks. We propose their analogues under central differential privacy. Our results are simple, easy to apply and we use them to establish sample complexity bounds in several estimation tasks. 

\noindent{We establish the optimal sample complexity of discrete distribution estimation under total variation distance and $\ell_2$ distance. We also provide lower bounds for several other distribution classes, including product distributions and Gaussian mixtures that are tight up to logarithmic factors. The technical component of our paper relates coupling between distributions to the sample complexity of estimation under differential privacy. 
}
\end{abstract}

\section{Introduction}
{
Statistical estimation tasks are often characterized by the optimal trade-off between the sample size and estimation error. There are two steps in establishing tight sample complexity bounds: An information theoretic lower bound on sample complexity and an algorithmic upper bound that achieves it.}
Several works 
have developed general tools to obtain the lower bounds (e.g.,~\cite{LeCam73, Assouad83, IbragimovK13, BickelR88, devroye1987course, verdu1994generalizing, CoverT06, ScarlettC19}, and references therein), {and three prominent techniques are Le Cam's method, Fano's inequality, and Assouad's lemma.} An exposition of {these three methods and their connections} is presented in~\cite{Yu97}\footnote{The title of~\cite{Yu97}, ``Assouad, Fano, and Le Cam'' is the inspiration for our title.}. 


In several estimation tasks, individual samples have sensitive information that must be protected. 
This is particularly of concern in applications such as healthcare, finance, geo-location, etc. Privacy-preserving computation has been studied in various fields including database, cryptography, statistics and machine learning~\cite{Warner65, Dalenius77, DinurN03, WassermanZ10, DuchiJW13, ChaudhuriMS11}. Differential privacy (DP)~\cite{DworkMNS06}, which allows statistical inference while preserving the privacy of the individual samples, has become one of the most {popular notions of privacy}~\cite{DworkMNS06, WassermanZ10, DworkRV10, BlumLR13, McSherryT07, DworkR14, KairouzOV17}. Differential privacy has been adopted by the US Census Bureau for the 2020 census and several large technology companies, including Google, Apple, and Microsoft~\cite{ErlingssonPK14, AppleDP17, DingKY17}.

\medskip
\noindent\textit{Differential privacy}~\cite{DworkMNS06}. {Let $\cX$ denote an underlying data domain of individual data samples} and {$\cX^\ns$ be the set of all possible length-$n$ sequences over $\cX$}. For $x, y\in \cX^\ns$, $\ham{x}{y}$ is their Hamming distance, the number of coordinates they differ at. A (randomized) estimator $\hat\theta:\cX^\ns\to\Theta$ is \emph{$(\eps,\delta)$-differentially private} (denoted {as} $(\eps,\delta)$-DP) if for any $S \subseteq \Theta$, and all $x, y\in\cX^\ns$ with $\ham{x}{y}\le 1$, the following holds:
	\begin{align}
		\probof{\hat\theta(x)\in S}\le e^{\eps} \cdot \probof{\hat\theta(y)\in S}+\delta.\label{eqn:dp-definition}
	\end{align}

The case $\delta=0$ is \emph{pure differential privacy} and is denoted by $\eps$-DP. We consider problems of parameter estimation and {goodness-of-fit (hypothesis testing)} under differential privacy constraints.
\medskip

\noindent\textbf{Setting.} Let $\cP$ be \emph{{any}} collection of distributions over $\cX^\ns$, {where $\ns$ denotes the number of samples.}\footnote{In the general setting, we are not assuming i.i.d. distribution over $\cX^\ns$, although we will specialize to this case later.} {Let $\theta:\cP\to\Theta$ be a parameter of the distribution that we want to estimate}. {Let $\ell:\Theta\times\Theta\to\RR_+$ be a pseudo-metric which is the loss function for estimating $\theta$. We now describe the minimax framework of parameter estimation, and hypothesis testing.}

\medskip

\noindent\textit{Parameter estimation.} The risk of an estimator $\hat{\theta}:\cX^\ns\to\Theta$ under loss $\ell$ is $\max_{\p \in \cP} \EE_{X \sim \p}\left[\ell(\hat \theta (X),\theta(p))\right]$, {the worst case expected loss of $\hat\theta$ over $\cP$. Note that $X\in\cX^\ns$, since $\p$ is a distribution over $\cX^\ns$. The minimax risk of estimation under $\ell$ for the class $\cP$ is}
\begin{align*}
R(\cP, \ell) := \min_{\hat \theta}\ \max_{\p \in \cP}\ \EE_{X \sim \p}\left[\ell(\hat \theta (X),\theta(p))\right].
\end{align*}
{The minimax risk under differentially private protocols is given by restricting $\hat\theta$ to be differentially private. For $(\eps, \delta)$-DP, we study the following minimax risk:} 
\begin{align}
R(\cP, \ell, \eps, \delta) := \min_{\hat \theta \text{ is }(\eps,\delta)\text{-DP}}\ \max_{\p \in \cP}\ \EE_{X \sim \p}\left[\ell(\hat \theta (X),\theta(p))\right].\label{eqn:minmax-risk}
\end{align}
For $\delta=0$, the above minimax risk under $\eps$-DP is denoted as $R(\cP, \ell, \eps)$. 

\medskip

\noindent\textit{Hypothesis testing.} Hypothesis testing can be cast in the framework of parameter estimation as follows. {Let $\cP_1\subset \cP$, and $\cP_2\subset \cP$ be two disjoint subsets of distributions denoting the two hypothesis classes. Let $\Theta=\{1,2\}$, such that for $\p\in\cP_i$, let $\theta(p)=i$. For a test $\hat\theta:\cX^n\to \{1,2\}$, and $\ell(\theta, \theta')=\II\{\theta\ne\theta'\}=|\theta-\theta'|$,} {the error probability is the worst case risk under this loss function:}
\begin{align}
P_e(\hat\theta, \cP_1, \cP_2):= \max_i \max_{\p\in \cP_i}\probof{\hat\theta(X)\ne i \mid X\sim \p} = \max_i \max_{\p\in \cP_i}\EE_{X\sim \p} \left[{|\hat\theta(X)-\theta(p)}|\right].\label{eqn:error-prob}
\end{align}



\smallskip
\noindent\textbf{Organization.} The remainder of the paper is organized as follows. In Section~\ref{sec:dp-lecam},~\ref{sec:dp-fano}, and~\ref{sec:dp-assouad} we state the privatized versions of Le Cam, Fano, and Assouad's method respectively. In Section~\ref{sec:applications} we discuss the applications of these results to several estimation tasks. In Section~\ref{sec:related} we discuss some related and prior work. In Section~\ref{sec:pdp_estimation} and~\ref{sec:adp_applications} we prove the bounds for distribution estimation under pure DP and approximate DP respectively. In Section~\ref{sec:proofs}, we prove the main theorems stated in Section~\ref{sec:results}.  
 
\section{Results}
\label{sec:results}
Le Cam's method is used to establish lower bounds for hypothesis testing and functional estimation. Fano's inequality, and Assouad's lemma prove {lower bounds} for multiple hypothesis testing problems and can be applied to parameter estimation tasks such as estimating distributions. We \newhzn{develop} extensions of these results with differential privacy.

\medskip

{
\noindent\textit{An observation.} 
A \emph{coupling} between distributions $\p_1$ and $\p_2$ over $\cX^\ns$ is a joint distribution $(X,Y)$ over $\cX^\ns \times\cX^\ns$ whose marginals satisfy $X \sim \p_1$ and $Y \sim \p_2$. 
Our lower bounds are based on the following observation. If there is a coupling $(X,Y)$ between distributions $\p_1$ and $\p_2$ over $\cX^\ns$ with $\expectation{\ham{X}{Y}}=D$, then a draw from $\p_1$ can be converted to a draw from $\p_2$ by changing $D$ coordinates in expectation. 
\newzs{By the group property of differential privacy, roughly speaking, for any $(\eps,\delta)$-DP estimator $\hat{\theta}$, it must satisfy $$\forall S\subseteq \Theta, \probof{\hat\theta(X) \in S| X \sim \p} \le e^{D\eps}\cdot \probof{\hat\theta(Y) \in S| Y \sim \q}+ \delta D e^{\eps(D-1)}.$$
	Hence, if there exists an algorithm that distinguishes between $\p_1$ and $\p_2$ reliably, $D$ must be large ($\Omega\Paren{1/\eps+\delta}$).}

\subsection{DP Le Cam's method}
\label{sec:dp-lecam}
\newhz{
Le Cam's method (Lemma 1 of~\cite{Yu97}) is widely used to prove lower bounds for composite hypothesis testings such as uniformity testing~\cite{Paninski08}, density estimation~\cite{Yu97, ray2016cam}, and estimating functionals of distributions~\cite{JiaoVHW15, WuY16, polyanskiy2019dualizing}.}


We use the expected Hamming distance between couplings of distributions in the two classes to obtain the following extension of Le Cam's method with $(\eps, \delta)$-DP, which is an adaptation of a similar result in~\cite{AcharyaSZ18}. For the hypothesis testing problem described above, let $\text{co}(\cP_i)$ be the convex hull of distributions in $\cP_i$, which are also families of distributions over $\cX^\ns$.

\begin{restatable}[$(\eps, \delta)$-DP Le Cam's method]{theorem}{dplecam}
	\label{thm:le_cam}
	Let $\p_1\in\text{co}(\cP_1)$ and $\p_2\in\text{co}(\cP_2)$. Let $(X,Y)$ be a coupling between $\p_1$ and $\p_2$ with $D=\expectation{\ham{X}{Y}}$. Then for $\eps\ge0, \delta\ge0$, any $(\eps,\delta)$-differentially private  hypothesis testing algorithm $\hat{\theta}$ must satisfy 
	\begin{align}
	P_e(\hat{\theta}, \cP_1, \cP_2)  \ge  \frac12 \max \left\{  1-\dtv{p_1}{p_2}, 0.9 e^{-10 \eps D} - 10 D \delta\right\},\label{eqn:le-cam}
	\end{align}
	{where $\dtv{p_1}{p_2} := \sup_{A \subseteq \cX^\ns} \Paren{p_1(A)-p_2(A)}=\frac12{\ell_1(\p_1, \p_2)}$ is the total variation (TV) distance of $p_1$ and $p_2$.}
\end{restatable}
The first term here is the original Le Cam's result~\cite{LeCam73, LeCam86, Yu97, Canonne15} and the second term is a lower bound on the additional error due to privacy. \newzs{Note that the second term increases when $D$ decreases. Choosing $\p_1, \p_2$ with small $D$ makes the RHS of~\eqref{eqn:le-cam} large, hence giving better testing lower bounds.} {A similar result (Theorem 1 in~\cite{AcharyaSZ18})}, along with a suitable coupling was used in~\cite{AcharyaSZ18} to obtain the optimal sample {complexity of testing discrete distributions}. We defer the proof of this theorem to Section~\ref{sec:le}.


\subsection{DP Fano's inequality}
\label{sec:dp-fano}
\newzs{Theorem~\ref{thm:le_cam} (DP Le Cam's method) characterizes lower bounds for binary hypothesis testing. In estimation problems with multiple parameters, it is common to reduce the problem to a multi-way hypothesis testing problem. The following theorem, proved in Section~\ref{sec:fano_pro}, provides a lower bound on the risk of multi-way hypothesis testing under $\eps$-DP.}
Let 
\[
D_{KL}\Paren{\p_i,\p_j}:=\sum_{x\in\cX^\ns}p_i(x)\log \frac{p_i(x)}{p_j(x)}
\]
 be the $KL$ divergence between (discrete) distributions $\p_i$ and $\p_j$. \footnote{\newzs{For continuous distributions, the summation is replaced with an integral and the probability mass functions are replaced with densities. We focus on discrete distributions in the proof while the results hold for continuous distributions as well.}}
\begin{restatable} [$\eps$-DP Fano's inequality] {theorem}{dpfano}
	\label{thm:dp_fano}
	Let $\cV=\{\p_1, \p_2,...,\p_M\}\subseteq \cP$ such that {for all $i\ne j$,} 
	\begin{enumerate}[label=(\alph*)]
		\item  $\ell \Paren{\theta(\p_i),\theta(\p_j)} \ge \alpha$,
		\item $D_{KL} \Paren{\p_i,\p_j} \le \beta$,
		\item \label{item:dham} there exists a coupling $(X, Y)$ between $\p_i$ and $\p_j$ such that $\expectation{\ham{X}{Y}} \le D$, then
	\end{enumerate}
	\begin{align} \label{eqn:fano_result}
		R(\cP, \ell, \eps) \ge \max \Bigg\{ & \frac{\alpha}{2} \left(1 - \frac{\beta + \log 2}{\log M}\right), 0.4\alpha \min\left\{1, \frac{M}{e^{10\eps D}}\right\} \Bigg\}. 
	\end{align}
\end{restatable}

\newzs{Under pure DP constraints, Theorem~\ref{thm:dp_fano} extends Theorem~\ref{thm:le_cam} to the multiple hypothesis case.} Non-private Fano's inequality (e.g., Lemma 3 of~\cite{Yu97}) requires only conditions $(a)$ and $(b)$ and provides the first term of the risk bound above. {Now, if we consider the second term, which is the additional cost due to privacy, we would require $\exp(10\eps D)\ge M$, \ie $D \ge {\log M}/{(10\eps)}$ to achieve a risk less than $0.4\alpha$.} Therefore, for reliable estimation, the expected Hamming distance between any pair of distributions cannot be too small. In Corollary~\ref{coro:fano}, we provide a corollary of this result to establish sample complexity lower bounds for several distribution estimation tasks.

	Theorem~\ref{thm:dp_fano} ($\eps$-DP Fano's inequality) can also be seen as a probabilistic generalization of the classic packing lower bound~\cite{Vadhan17}. The packing argument, with its roots in database theory, considers inputs to be deterministic datasets whose pairwise Hamming distances are bounded with probability one, while Theorem~\ref{thm:dp_fano} considers randomly generated datasets whose Hamming distances are bounded in expectation. This difference makes Theorem~\ref{thm:dp_fano} better suited for proving lower bounds for statistical estimation problems. We discuss this difference in details in Section~\ref{sec:related}.

\smallskip
\noindent\textbf{Remark.} {Theorem~\ref{thm:dp_fano} is a bound on the risk for pure differential privacy ($\delta=0$). Our proof extends to $(\eps,\delta)$-DP for \newhz{$\delta = O \Paren{\frac1M}$},  which is not sufficient to establish meaningful bounds since in most problems $M$ will be chosen to be exponential in the problem parameters. To circumvent this difficulty, in the next section we provide a private analogue of Assouad's method, which also works for $(\eps,\delta)$-DP.}




\subsection{DP Assouad's method}
\label{sec:dp-assouad}
{Our next result is a private version of Assouad's lemma (Lemma 2 of~\cite{Yu97}, and~\cite{Assouad83}). Recall that $\cP$ is a set of distributions over $\cX^\ns$. Let $\cV\subseteq\cP$ be} a set of distributions indexed by the hypercube $ \cE_{\ab} := \{\pm 1\}^{\ab}$, and the loss  $\ell$ is such that
\begin{align}
\forall u,v \in \cE_{\ab}, \ell(\theta(\p_u), \theta(\p_v)) \ge 2\tau \cdot \sum_{i=1}^{\ab} \mathbb{I}\Paren{u_i\neq v_i}.
\label{eqn:assouad-loss}
\end{align}

{Assouad's method provides a lower bound on the estimation risk for distributions in $\cV$, which is a lower bound for $\cP$. For each coordinate $i\in[k]$, consider the following mixture distributions obtained by averaging over all distributions with a fixed value at the $i$th coordinate,}
\[
	\p_{+i}  =  \frac{2}{|\cE_{k}|} \sum_{e \in \cE_{k}: e_i= + 1}\p_e, ~~~ \p_{-i}  =  \frac{2}{|\cE_{k}|} \sum_{e \in \cE_{k}: e_i= -1}\p_e.
\]
Assouad's lemma provides a lower bound on the risk by using~\eqref{eqn:assouad-loss} and considering the problem of distinguishing $\p_{+i}$ and $\p_{-i}$. 
{Analogously, we prove the following privatized version of Assouad's lemma by considering the minimax risk of a private hypothesis testing $\phi:\cX^\ns\to\{-1,+1\}$ between $\p_{+i}$ and $\p_{-i}$. The detailed proof is in Section~\ref{sec:private_assouad}.
}
\begin{restatable} [DP Assouad's method] {theorem} {dpassouad}
	\label{thm:assouad} 
{$\forall i \in [k]$, let $\phi_i:\cX^\ns\to\{-1,+1\}$ be a binary classifier.
	\begin{align*}
	R(\cP, \ell, \eps, \delta) \ge \frac{\tau}{2} \cdot \sum_{i=1}^{\ab} \min_{\phi_i \text{ is $(\eps, \delta)$-DP}}   ( \probofsub{ X\sim \p_{+i} }{ \phi_i(X) \neq 1 } 
	+ \probofsub{X\sim \p_{-i} }{\phi_i(X) \neq -1 } ). \nonumber
	\end{align*}
	Moreover, if $\forall i\in [\ab]$, there exists a coupling $(X,Y)$ between $\p_{+i}$ and $\p_{-i}$ with $\expectation{\ham {X} {Y}} \le D$,
	\begin{equation} \label{eqn:assouad}
	R(\cP, \ell, \eps, \delta) \ge \frac{\ab \tau}{2} \cdot \Paren{0.9 e^{-10 \eps D} - 10 D \delta}. 
	\end{equation} }
\end{restatable}

%
\noindent The first bound is the classic Assouad's Lemma and~\eqref{eqn:assouad} is the loss due to privacy constraints. Once again note that~\eqref{eqn:assouad} grows with decreasing $D$. \newzs{Compared to Theorem~\ref{thm:dp_fano} (DP Fano's inequality), Theorem~\ref{thm:assouad} works under $(\eps, \delta)$-DP, which is a less stringent privacy notion.}

\subsection{Applications}
\label{sec:applications}
We now describe several applications of the theorems above. 

\medskip
\noindent\textbf{Applications of Theorem~\ref{thm:le_cam}.} 
~\cite{AcharyaSZ18} developed a result similar to Theorem~\ref{thm:le_cam}, which is used to establish sample complexity lower bounds for differentially private uniformity testing under total variation distance~\cite{AcharyaSZ18, AliakbarpourDR18}, and for differentially private entropy and support size estimation~\cite{AcharyaKSZ18}. In this paper, we use Theorem~\ref{thm:le_cam} as a stepping stone to prove private Assouad's method (Theorem~\ref{thm:assouad}). 

\medskip
\noindent\textbf{Distribution Estimation and Applications of Theorem~\ref{thm:dp_fano} and~\ref{thm:assouad}.} {We will apply Theorem~\ref{thm:dp_fano} (private Fano's inequality) and Theorem~\ref{thm:assouad} (private Assouad's lemma) to some classic distribution estimation problems. The results are summarized in Table~\ref{tab:pure} and Table~\ref{tab:approx}. Before presenting the results, we describe the framework of minimax distribution estimation.}

\medskip
\noindent\textit{Distribution estimation framework.} Let $\cQ$ be a collection of distributions over $\cX$, and for this $\cQ$, let $\cP=\cQ^{\ns}:=\{q^\ns:q\in\cQ\}$ be the collection of $n$-fold distributions over $\cX^\ns$ induced by i.i.d. draws from a distribution over $\cQ$. The parameter space is $\Theta=\cQ$, where $\theta(\q^n)=\q$, and $\ell$ is a distance measure between distributions in $\cQ$. 
Let $\alpha>0$ be a fixed parameter. The sample complexity,  $S(\cQ, \ell, \alpha, \eps, \delta)$ is the smallest number of samples $n$ to make $R(\cQ^\ns, \ell, \eps, \delta)\le \alpha$, \ie
\[
S(\cQ, \ell, \alpha, \eps, \delta)= \min \{n: R(\cQ^\ns, \ell, \eps, \delta)\le \alpha\}.
\]
When $\delta = 0$, we denote the sample complexity by $S(\cQ, \ell, \alpha, \eps)$. We will state our results in terms of sample complexity. The following corollary of Theorem~\ref{thm:dp_fano} can be used to prove lower bounds on the sample complexity in this distribution estimation framework.  The detailed proof of the corollary is in Section~\ref{sec:coro_fano_proof}.

\begin{restatable}[$\eps$-DP distribution estimation]{corollary} {distest}
	\label{coro:fano}
	Given $\eps>0$, let $\cV=\{\q_1, \q_2,...,\q_M\} \subseteq \cQ$ be a set distributions over $\cX$ with size $M$, such that for all $i\ne j$, 	\begin{enumerate}[label=(\alph*)]
		\item $\ell \Paren{\q_i, \q_j} \ge 3 \tau$,
		\item $D_{KL} \Paren{\q_i,\q_j} \le \beta$,
		\item $d_{TV} \Paren{\q_i,\q_j} \le \gamma$,
	\end{enumerate}
	then
\[
		S(\cQ, \ell, \tau, \eps) = \Omega\Paren{\frac{\log M}{\beta}+ \frac{\log M}{\gamma \eps}}.
\]
\end{restatable}

\noindent\textbf{Remark.} With only conditions $(a)$ and $(b)$, we obtain the first term of the sample complexity lower bound which is the original Fano's bound for sample complexity.

\begin{table*}[htb]
\centering
      \begin{tabular}{| c | c | c |}
      \hline
      {\bf Problem} & {\bf Upper Bounds} & {\bf Lower Bounds} \\ \hline
      {\bf $\ab$-ary} & \multicolumn{2}{c|} {\parbox[c][1.2cm]{10cm}{\centering$ \Theta \Paren{\frac{\ab}{\alpha^2}+\frac{\ab}{\alpha\eps}}$~(\cite{DiakonikolasHS15}, Theorem~\ref{thm:pure-dv})} }\\\hline
 
       {\bf $\ab$-ary, $\ell_2$ distance} &\parbox[c][1.5cm]{5cm}{\centering $O\Paren{\frac{1}{\dist^2}+\min\Paren{\frac{\sqrt{\ab}}{\dist\eps},\frac{\log\ab}{\dist^2\eps}  } }$ \\\small{(Theorem~\ref{thm:pure-l2}) }}
       & \parbox[c]{5cm}{\centering $\Omega \Paren{\frac{1}{\dist^2}+\min\Paren{\frac{\sqrt{\ab}}{\dist\eps},\frac{\log (\ab \dist^2)}{\dist^2\eps}  } }$ \\\small{(Theorem~\ref{thm:pure-l2}) } } \\ \hline
       
       {\bf product distribution} & \parbox[c][1.5cm]{5cm}{\centering $O\Paren{\ab\dims \log \Paren{ \frac{\ab\dims}{\dist}} \Paren{\frac1{\dist^2}+\frac1{\dist\eps}}}$ \\\small{\cite{BunKSW2019} }}
       &\parbox[c]{5cm}{\centering $\Omega \Paren{{\ab\dims}\Paren{\frac1{\alpha^2}+\frac1{\alpha \eps}}}$ \\\small{(Theorem~\ref{thm:main_product}) }}\\\hline
       
              {\bf Gaussian mixtures} & \parbox[c][1.5cm]{5cm}{\centering ${O}\Paren{kd\log(\frac{dR}{\alpha})(\frac1{\alpha^2}+ \frac1{\alpha \eps})}$ \\\small{\cite{BunKSW2019} }}
       &\parbox[c]{5cm}{\centering $\Omega \Paren{{\ab\dims}\Paren{\frac1{\alpha^2}+\frac1{\alpha \eps}}}$ \\\small{(Theorem~\ref{thm:main_Gaussian}) }}\\\hline
       
      \end{tabular}
    \caption{\label{tab:pure} Summary of the sample complexity bounds for $\eps$-DP discrete distribution estimation. Unless mentioned, the bounds are all for estimation under total variation distance.}
\end{table*}

%

\begin{table*}[htb]
\centering
      \begin{tabular}{| c | c | c |}
      \hline
      {\bf Problem} & {\bf Upper Bounds} & {\bf Lower Bounds} \\ \hline
      {\bf $\ab$-ary} &  \parbox[c][1.5cm]{5cm}{\centering $O\Paren{\frac{\ab}{\dist^2}+\frac{\ab} {\dist\eps}}$ \\\small{(\cite{DiakonikolasHS15}, Theorem~\ref{thm:ApproximateDiscreteTotalVariation})} }
       &\parbox[c]{5cm}{\centering $\Omega \Paren{\frac{\ab}{\alpha^2}+\frac{\ab}{\alpha (\eps+\delta)}}$ \\\small{(Theorem~\ref{thm:ApproximateDiscreteTotalVariation})} }\\\hline
 
       {\bf $\ab$-ary, $\ell_2$ distance} &\parbox[c][1.5cm]{5cm}{\centering $O \Paren{\frac{1}{\dist^2}+\min\Paren{\frac{\sqrt{\ab}}{\dist},\frac{\log\ab}{\dist^2\eps}  } }$ \\\small{(Theorem~\ref{thm:ApproximateDiscreteL2}) }}
       & \parbox[c]{5cm}{\centering $\Omega \Paren{\frac{1}{\dist^2}+\min\Paren{\frac{\sqrt{\ab}}{\dist(\eps+\delta)},\frac{1}{\dist^2(\eps+\delta)}  } }$ \\\small{(Theorem~\ref{thm:ApproximateDiscreteL2}) } } \\ \hline
       
              {\makecell { \bf product distribution \\ ($k = 2$)}} & \parbox[c][1.5cm]{5cm}{\centering $O\Paren{\dims \log \Paren{ \frac{\dims}{\dist}} \Paren{\frac1{\dist^2}+\frac1{\dist\eps}}}$ \\\small{\cite{KamathLSU18, BunKSW2019}}}
       &\parbox[c]{5cm}{\centering $\Omega \Paren{\frac{\dims}{\alpha^2}+\frac{\dims}{\alpha (\eps + \delta)}}$ \\\small{(Theorem~\ref{thm:ApproximateProductTotalVariation},~\cite{KamathLSU18}) }}\\ \hline
      \end{tabular}
    \caption{\label{tab:approx} Summary of the sample complexity bounds for $(\eps,\delta)$-DP discrete distribution estimation. Unless mentioned, the bounds are all for estimation under total variation distance.}
\end{table*}

\medskip

\medskip
\noindent We now present examples of distribution classes we consider.

\medskip
\noindent\textit{$k$-ary discrete distribution estimation.} {Suppose $\cX=[k]:=\{1, \ldots, \ab\}$, and $\cQ:=\Delta_k$ is the simplex of $k$-ary distributions over $[k]$.  We consider estimation in both total variation and $\ell_2$ distance.} 

\medskip
\noindent\textit{$(k,d)$-product distributions.} Consider $\cX=[k]^d$, and let $\cQ:=\Delta_k^d$ be the set of product distributions over $[k]^d$, where the marginal distribution on each coordinate is over $[k]$ and independent of the other coordinates. We study estimation under total variation distance. A special case of this is Bernoulli product distributions ($k=2$), where each of the $d$ coordinates is an independent Bernoulli random variable.  

\medskip
\noindent\textit{$d$-dimensional Gaussian mixtures.} Suppose $\cX=\RR^{\dims}$, and $\cG_{\dims} := \{ \cN (\mu, I_d): \norm{\mu}\le R\}$ is the set of all Gaussian distributions in $\RR^d$ with bounded mean and identity covariance matrix. The bounded mean assumption is unavoidable, since by~\cite{BunKSW2019}, it is not possible to learn a single Gaussian distribution under pure DP without this assumption. {We consider 
\[
\cQ = \cG_{\ab, \dims} :=\left\{\sum_{j=1}^{\ab} w_j \p_j: \forall j \in [k], w_j \ge 0, \p_j\in\cG_\dims, w_1+\ldots+w_k =1\right\},
\] the collection of mixtures of $k$ distributions from $\cG_\dims$.} 

\medskip
\noindent\textbf{Applications of Theorem~\ref{thm:dp_fano}.}
We apply Corollary~\ref{coro:fano} and obtain sample complexity lower bounds for the tasks mentioned above under pure  differential privacy.

\medskip
{\noindent\textit{$\ab$-ary distribution estimation.} Without privacy constraints, the sample complexity of $\ab$-ary discrete distributions under total variation, and $\ell_2$ distance is $\Theta(k/\alpha^2)$ and $\Theta(1/\alpha^2)$ respectively, achieved by the empirical estimator. Under $\eps$-DP constraint, an upper bound of $O \Paren{\ab/\alpha^2+\ab/\alpha\eps}$
samples for total variation distance is known using Laplace mechanism~\cite{DworkMNS06} (e.g.~\cite{DiakonikolasHS15}). In Theorem~\ref{thm:pure-dv}, we establish the sample complexity of this problem by providing a lower bound that matches this upper bound.} 

Under $\ell_2$ distance, in Theorem~\ref{thm:pure-l2} we design estimators and establish their optimality whenever $\dist< \ab^{-1/2}$ or $\dist \ge \ab^{-0.499}$, which contains almost all the parameter range. Note that under $\ell_2$ distance, estimation without privacy has sample complexity independent of $k$, whereas {an unavoidable} logarithmic dependence on $\ab$ is introduced due to privacy requirements. The results are presented in Section~\ref{sec:pdp_kary}. 
 
\medskip
\noindent\textit{$(\ab,\dims)$-product distribution estimation.} For $(\ab,\dims)$-product distribution estimation under $\eps$-DP,~\cite{BunKSW2019} proposed an algorithm that uses
$O\Paren{\ab\dims \log \Paren{ \ab\dims/\dist} \Paren{1/\dist^2+1/\dist\eps}}$ samples. 
In this paper, we present a lower bound of $\Omega\Paren{{\ab\dims}/{\dist^2}+{\ab\dims}/{\dist\eps}}$, which matches their upper bound up to logarithmic factors. For Bernoulli product distributions, \cite{KamathLSU18} proved a lower bound of $\Omega\Paren{{d}/{\alpha \eps}}$ under $(\eps,{3}/{64\ns})$-DP, which is naturally a lower bound for pure DP. 
The details are presented in Section~\ref{sec:product}.

\medskip
\noindent\textit{Estimating Gaussian mixtures.}{~\cite{BunKSW2019} provided an upper bound of $\widetilde{O}\Paren{{\ab\dims}/{\dist^2}+{\ab\dims}/{\dist\eps}}$ samples.} Without privacy, a tight bounds of $\Omega({\ab\dims}/{\alpha^2})$ was shown in~\cite{SureshOAJ14, DaskalakisK14, AshtianiBHLMP18}. {In this paper, we prove a lower bound of $\Omega\Paren{{\ab\dims}/{\dist^2}+{\ab\dims}/{\dist\eps}}$, which matches the upper bound up to logarithmic factors.} \newzs{For the special case of estimating a single Gaussian ($k=1$), a lower bound of $\ns = {\Omega}\Paren{{\dims}/{(\alpha\eps\log \dims)}}$ was given in~\cite{KamathLSU18} for $(\eps,{3}/{64\ns})$-DP, which implies a lower bound that is $\log d$ factor weaker than our result under pure DP.}


\medskip
\noindent\textbf{Applications of Theorem~\ref{thm:assouad}.} As remarked earlier, Theorem~\ref{thm:dp_fano} only works for pure DP (or approximate DP with very small $\delta$). 
Assouad's lemma can be used to obtain lower bounds for distribution estimation under $(\eps,\delta)$-DP. For $k$-ary distribution estimation under $TV$ distance, we get a lower bound of $\Omega\Paren{{k}/{\alpha^2} + {k}/{\alpha(\eps + \delta)}}$. This shows that even up to $\delta=O(\eps)$, the sample complexity for $(\eps,\delta)$-DP is the same as that under $\eps$-DP.

\newzs{For Bernoulli ($k = 2$) product distributions, \cite{KamathLSU18} provides an efficient $(\eps, \delta)$-DP algorithm that achieves an upper bound of $O\Paren{\dims \log \Paren{ \dims/\dist} \Paren{1/\dist^2+1/\dist\eps}}$.\footnote{\newzs{The algorithm in~\cite{BunKSW2019} works for $\eps$-DP and general $k$ but it is not computationally efficient.}}} The lower bound $\Omega ({d}/{\alpha^2} + {d}/{\alpha \eps})$ obtained in~\cite{KamathLSU18} by fingerprinting holds for small values of $\delta = O(1/\ns)$. Note by the definition of DP~\eqref{eqn:dp-definition}, if $\delta>1/\ns$, a DP algorithm can blatantly disregard the privacy of $\delta\ns$ users. Therefore in most of the literature, $\delta$ is assumed to be $O(1/\ns)$. We want to make a complimentary remark that we
can obtain the same lower bound all the way up to $\delta = O(\eps)$. This shows that there is no gain even if we compromise the privacy of a $\delta$ fraction of users. Therefore, there is no incentive to do it. We describe the details about these applications in Section~\ref{sec:adp_applications}. 

\subsection{Related and prior work}
\label{sec:related}
\subsubsection{Private distribution estimation}
Protecting privacy generally comes at the cost of performance degradation. Previous literature has studied various problems and established utility privacy trade-off bounds, including distribution estimation, hypothesis testing, property estimation, empirical risk minimization, etc~\cite{ChaudhuriMS11, Lei11, BassilyST14, DiakonikolasHS15, CaiDK17, AcharyaSZ18, KamathLSU18, AliakbarpourDR18, AcharyaKSZ18}. 

There has been significant recent interest in differentially private distribution estimation.~\cite{DiakonikolasHS15} gives upper bounds for privately learning $\ab$-ary distributions under total variation distance.~\cite{KamathLSU18, BunKSW2019, KarwaV18} focus on high-dimensional distributions, including product distributions and Gaussian distributions. As discussed in the previous section, our proposed lower bounds improve upon their lower bounds in various settings. \cite{BunNSV15} studies the problem of privately estimating a distribution in Kolmogorov distance, which is weaker than total variation distance. Upper and lower bounds for private estimation of the mean of product distributions in $\ell_\infty$ distance, heavy tailed distributions, and Markov Random fields are studied in~\cite{BlumDMN05, DworkMNS06, SteinkeU17a, BunUV18, KamathSU20, ZhangKKW20}.

Several estimation tasks including distribution estimation and hypothesis testing have also been considered under the distributed notion of local differential privacy, e.g.,~\cite{Warner65, KasiviswanathanLNRS11, ErlingssonPK14, DuchiJW13, KairouzBR16, WangHWNXYLQ16, Sheffet17, YeB18, GaboardiR17, AcharyaSZ18a, AcharyaS19, AcharyaCFT18}.

\subsubsection {Lower bounds in differential privacy}
Several methods have been proposed in the literature to prove lower bounds {under DP constraints}. These include packing argument~\cite{HardtT10,Vadhan17, BeimelKN10}, fingerprinting~\cite{BunNSV15, SteinkeU17a, SteinkeU15, BunSU17, BunUV18, KamathLSU18} and coupling based arguments~\cite{AcharyaSZ18, KarwaV18}.

\smallskip
\noindent\textit{Binary Testing and Coupling.} Coupling based arguments have been recently used to prove lower bounds for binary hypothesis testing, including the independent works of~\cite{AcharyaSZ18, KarwaV18}.~\cite{AcharyaSZ18} establishes a very similar result to Theorem~\ref{thm:le_cam} and uses it to obtain lower bounds for a composite hypothesis testing problem on discrete distributions.~\cite{KarwaV18} proves a similar result for simple hypothesis testing and uses it to lower bound the sample complexity of estimating the mean of a one-dimensional Gaussian distribution. 
For both papers, the coupling argument implies that it is hard to differentially privately distinguish between two distributions, supposing there exists a coupling with small expected Hamming distance. This method can be viewed as another form of private Le Cam's method (Theorem~\ref{thm:le_cam}) and it can only be applied where binary hypothesis testing is involved. 
\newzs{\cite{CanonneKMSU19} uses coupling bounds in~\cite{AcharyaSZ18} to derive instance-optimal bounds for simple binary hypothesis testing under pure DP. They consider a coupling only for symbols whose likelihood ratio between the two hypothesis distributions is large, which results in better bounds for certain instances. The argument only considers pure DP and the case where samples are i.i.d generated while Theorem~\ref{thm:le_cam} and~\cite{AcharyaSZ18} can handle approximate DP and arbitrary distributions (e.g. mixtures of i.i.d distributions considered in this paper) .}
\smallskip

%
%

\noindent\textit{Pure DP Estimation and Packing.} Packing argument~\cite{HardtT10, BeimelKN10, Vadhan17} is a geometric approach to prove lower bounds for estimation under pure DP. We state a form of the packing bound below:
\begin{lemma}[Packing lower bound~\cite{Vadhan17}] \label{thm:packing}
	Let $\cV=\{x_1, x_2,...,x_M\}$ be a set of $M$ datasets over $\cX^n$.  For any pair of datasets $x_i$ and $x_j$, we have $\ham{x_i}{x_j} \le d$. Let $ \{S_{i}\}_{i \in [M]}$ be a collection of disjoint subsets of $\cS$. If there exists an $\eps$-DP algorithm $\cA : \cX^{\ns} \to \cS$ such that for every $i \in [M]$, $\probof{\cA(x_i) \in S_{i}}\ge {9}/{10}$, then
	\[
	\eps = \Omega \Paren{\frac{\log M}{d}}.
	\]
\end{lemma}

Our $\eps$-DP Fano's inequality (Theorem~\ref{thm:dp_fano}) and its corollary for distribution estimation (Corollary~\ref{coro:fano}) can be viewed as a probabilistic packing argument which generalizes Lemma~\ref{thm:packing} to the case where $\cV$ consists of distributions over $\cX^n$ instead of deterministic datasets. The distances between distributions are measured in the minimum expected hamming distance between random datasets generated from a coupling between the distributions. Lemma~\ref{thm:packing} can be obtained from $\eps$-DP Fano's inequality by setting the distributions to be point masses over $\cX^n$. 

Note that $d$ in Lemma~\ref{thm:packing} is an upper bound on the worst-case Hamming distance while $D$ is a bound on the expected Hamming distance and therefore $D \le d$. In statistical applications where $D \ll d$, we can obtain stronger lower bounds by replacing $d$ with $D$. For example, in the $\ab$-ary distribution estimation problem, a naive application of the packing argument can only give a lower bound of $\ns = \Omega\Paren{{\ab \log \Paren{1/\dist}}/{\eps}}$ instead of the optimal $\ns = \Omega\Paren{{\ab}/{\dist \eps}}$ lower bound, where there is an exponential gap in the parameter $1/\dist$.

\smallskip

\noindent\textit{Approximate DP and Fingerprinting.} Fingerprinting~\cite{SteinkeU15, BunNSV15, DworkSSUV15, SteinkeU17a,  BunSU17, BunUV18, KamathLSU18, CaiWZ19} is a versatile lower bounding method for $(\eps,\delta)$-DP for $\delta=O(1/\ns)$. It has been used to prove lower bounds for several problems, including attribute mean estimation in databases~\cite{SteinkeU17a}, lower bounds on the number of online statistical queries~\cite{BunSU17}, and private selection problem~\cite{SteinkeU17b}.~\cite{KamathLSU18} uses fingerprinting to prove lower bounds on estimating Bernoulli product distributions and Gaussian distributions. We believe fingerprinting and DP Assouad's lemma are both powerful tools for proving lower bounds under approximate DP. In estimating Gaussian distributions, fingerprinting provides strong lower bounds under approximate DP, whereas private Assouad's method gives an additional polynomial blow-up compared to fingerprinting. 
However, for discrete distribution estimation, private Assouad's method provides tight lower bounds, and we do not know how to obtain such bounds from the fingerprinting lemma.



\cite{DuchiJW13} derives analogues of Le Cam, Assouad, and Fano in the local model of differential privacy, and uses them to establish lower bounds for several problems under local differential privacy.~\cite{AcharyaCT19, acharya2020interactive} proves lower bounds for various testing and estimation problems under local differential privacy using a notion of chi-squared contractions based on Le Cam's method and Fano's inequality. 

%
%
%
%
%
%



\section{$\eps$-DP distribution estimation} \label{sec:pdp_estimation}
In this section, we use Corollary~\ref{coro:fano} to prove sample complexity lower bounds for various $\eps$-DP distribution estimation problems. The general idea is to construct a subset of distributions in $\cQ$ such that they are close in both $TV$ distance and $KL$ divergence while being separated {in the loss function $\ell$}. The larger the subsets we construct, {the better} the lower bounds we can get. In Section~\ref{sec:pdp_kary}, we {derive} sample complexity lower bounds for $k$-ary distribution estimation under both \emph{TV} and $\ell_2$ distance that are tight up to constant factors. Tight sample complexity lower bounds up to logarithmic factors for $(k,d)$-product distributions and $d$-dimensional Gaussian mixtures are derived in Section~\ref{sec:product} and~\ref{sec:gaussian} respectively. 

{Corollary~\ref{coro:fano} requires a \emph{packing} of distributions with pairwise distance at least $3\tau$ apart in $\ell$. A standard method to construct such distributions is using results from coding theory.}

\newzs{We start with some definitions. An $h$-ary code of length $k$ is a set $\cC\subseteq \{0,1, \ldots, h-1\}^k$, and each $c\in\cC$ is a \emph{codeword}. The \emph{minimum distance} of a code $\cC$ is the smallest Hamming distance between two codewords in $\cC$. The code is called binary when $h = 2$. The weight of a binary codeword $c\in\cC$ is $wt(c)= |\{i:c_i=1\}|$, the number of 1's in $c$. A binary code $\cC$ is a \emph{constant weight code} if each $c\in\cC$ has the same weight.  We now present some useful variants of the classic Giblert Varshamov bounds on the existence of codes with certain properties. We prove these in Section~\ref{sec:codes}.
}
{
\begin{lemma}
\label{lem:GV}
Let $l$ be an integer at most $k/2$ and at least $20$. There exists a constant weight binary code $\cC$ which has code length $k$, weight $l$, minimum distance $l/4$ with $|\cC|\ge \Paren{\frac{k}{2^{7/8}l}}^{7l/8}$.
\end{lemma}} 

\begin{lemma}
\label{lem:constantGV2}
There exists an $h$-ary code $\cH$ with code length $\dims$ and minimum Hamming distance $\frac{\dims}{2}$, which satisfies that $\absv{\cH} \ge (\frac{h}{16})^{\frac{\dims}{2}}$.
\end{lemma}}

%
%

\subsection{$\ab$-ary distribution estimation} \label{sec:pdp_kary}
We establish the sample complexity of $\eps$-DP $\ab$-ary distribution estimation under $TV$ and $\ell_2$ distance. 
\begin{theorem}
\label{thm:pure-dv}
The sample complexity of $\eps$-DP $\ab$-ary distribution estimation under $TV$ distance  is
\begin{align}
S(\Delta_{\ab}, d_{TV}, \alpha, \eps)=\Theta \Paren{\frac{\ab}{\alpha^2}+\frac{\ab}{\alpha\eps}}.
\end{align}
\end{theorem}

\begin{theorem}
\label{thm:pure-l2}
The sample complexity of $\eps$-DP $\ab$-ary distribution estimation under $\ell_2$ distance is
\begin{align}
S(\Delta_{\ab}, \ell_2, \alpha, \eps) =\Theta \Paren{\frac{1}{\dist^2}+\frac{\sqrt{\ab}}{\dist\eps}},\ \ \ \text{for $\dist < \frac1{\sqrt{\ab}}$, and}
\end{align}
\begin{align}
\Omega\Paren{\frac{1}{\dist^2}+\frac{\log(\ab\dist^2)}{\dist^2\eps}}\le S(\Delta_{\ab}, \ell_2, \alpha, \eps) \le O\Paren{\frac{1}{\dist^2}+\frac{\log\ab}{\dist^2\eps}} \ \ \ \text{for $\dist > \frac1{\sqrt{\ab}}$.} 
\end{align}
\end{theorem}

For $\ell_2$ loss, our bounds are tight within constant factors when $\dist < \frac1{\sqrt{\ab}}$ or $\dist > \ab^{-(\frac12- 0.001)}$.

\subsubsection {Total variation distance}
\label{sec:puredp-tv}

In this section, we derive the sample complexity of $\eps$-DP $k$-ary distribution estimation under $TV$ distance, which is stated in Theorem~\ref{thm:pure-dv}.

\medskip

\noindent \textbf{Upper bound:}
\cite{DiakonikolasHS15} provides an upper bound based on Laplace mechanism~\cite{DworkMNS06}. {We state the algorithm and a proof for completeness and we will use it for estimation under $\ell_2$ distance.}

Given a $\Xon$ from an unknown distribution $\p$ over $[\ab]$. Let $M_x(\Xon)$ be the number of appearances of $x$ in $\Xon$. {Let $p^{\text{erm}}$ be the empirical estimator where $ p^{\text{erm}}(x) :=\frac{M_x(\Xon)}{\ns}$. We note that changing one $X_i$ in $X^n$ can change at most two coordinates of $p^{\text{erm}}$, each by at most $\frac1n$, and thus changing one $X_i$ changes the $p^{\text{erm}}$ by at most $2/n$ in $\ell_1$ distance. Therefore, by~\cite{DworkMNS06}, adding a Laplace noise of parameter $2/\ns\eps$ to each coordinate of $p^{\text{erm}}$ makes it $\eps$-DP. For $x\in[k]$, let
\[ 
h(x) = p^{\text{erm}}(x)+\text{Lap}\Paren{\frac{2}{n \eps}}, \]
where $Lap(\beta)$ is a Laplace random variable with parameter $\beta$.}
The final output $\hat{p}$ is the projection of $h$ on the simplex $\Delta_k$ {in $\ell_2$ distance}.
The expected $\ell_2$ loss between $h$ and $p$ can be upper bounded by
\begin{align}
\Paren{\expectation{\norm{h-p}}}^2 &\le \expectation{\norms{h-p}} \le \expectation{\norms{p^{\text{erm}}-p}}+\expectation{\norms{h - p^{\text{erm}}}}, \nonumber
\end{align}
where the first inequality comes from the Jensen's inequality and the second inequality comes from the triangle inequality.

The first term $\expectation{\norms{p^{\text{erm}} -p}}$ is upper bounded by $\frac{1}{n}$ by an elementary analysis of the empirical estimator. For the second term, note that $\expectation{\norms{h - p^{\text{erm}}}} = \sum_{i=1}^\ab \expectation{Z_i^2}$, where $\forall i, Z_i \sim \text{Lap}\Paren{\frac{2}{\ns\eps}}$. 
{By the variance of Laplace distribution}, we have $\expectation{\norms{p^{\text{erm}} -h}} = O\Paren{\frac{k}{n^2 \eps^2}}$.  Therefore $\expectation{\norm{h-p}}\le O\Paren{\frac1{\sqrt{\ns}}+\frac{\sqrt{k}}{\ns\eps}}$.

Note that since $\Delta_k$ is convex, $\norm{\hat{p}-p} \le\norm{h-p}$. Finally, by Cauchy-Schwarz Inequality, $\expectation{\absvn{\hat{p}- p}} \le \sqrt{\ab}\cdot \expectation{\norm{\hat{p}-\p}}  \le \sqrt{\ab}\cdot\expectation{\norm{h-\p}} = O\Paren{\sqrt{\frac{k}{n}}+ {\frac{k}{n \eps}}} $. Therefore $\expectation{\absvn{\hat{p}- p}} \le \alpha$ when $n= O \Paren{\frac{k}{\alpha^2}+\frac{k}{\alpha\eps}}$.

\medskip

\noindent \textbf{Lower bound.}
{We will construct a large set of distributions such that the conditions of Corollary~\ref{coro:fano} hold.} 
\ignore{Our lower bound is through the following two steps. The first step is to construct a set of distributions $\cV_\ab \subset \Delta_k$, such that each pair of distributions in $\cV_\ab$ is $\Omega(\dist)$-separated in total variation distance and with $|\cV_\ab|$ as large as possible. 
In the second step, we bound the $KL$ divergence and $TV$ distance between each pair of distributions, and apply Corollary~\ref{coro:fano} to get the sample complexity lower bound.}
{Suppose $\alpha < 1/48$. Applying Lemma~\ref{lem:GV} with $l={\ab}/{2}$, there exists a constant weight binary code $\cC$ of weight $k/2$, and minimum distance $k/8$, and $|\cC|>2^{7k/128}$. For each codeword $c\in\cC$, a distribution $\p_c$ over $[k]$ is defined as follows:
\[
\p_c (i)  = \left\{
\begin{array}{rcl}
\frac{1+24\dist}{\ab} ,       &      & \text{if ~ }c_i=1,\\
\frac{1-24\dist}{\ab},       &      & \text{if ~ }c_i=0.
\end{array} \right.
\]}

{We choose $\cV = \{ \p_c : c \in \cC \}$ to apply Corollary~\ref{coro:fano}. By the minimum distance property, any two distributions in $\cV$ have a total variation distance of at least $24\alpha/k\cdot k/8=3\alpha$, and at most $24\alpha$. Furthermore, by using $\log (1+x) \le x$, we can bound the KL divergence between distributions by their $\chi^2$ distance, 
\[d_{KL} (\p, \q) \le  \chi^2 (\p,\q)  = \sum_{x=1}^k \frac{\Paren{\p(x) -\q(x)}^2 }{\q(x)} < {10000\alpha^2}.
\]
Setting $\tau=\alpha$, $\gamma=24\alpha$, and $\beta=10000\alpha^2$, and using $\log M > 7\ab/64$ in Corollary~\ref{coro:fano}, we obtain $S(\Delta_{\ab}, d_{TV}, \alpha, \eps)= \Omega \Paren{\frac{k}{\alpha^2}+\frac{k}{\alpha\eps}}.$}
\subsubsection{$\ell_2$ distance}
In this section, we derive the sample complexity of $\eps$-DP $\ab$-ary distribution estimation under $\ell_2$ distance, which is stated in Theorem~\ref{thm:pure-l2}.

\medskip

\noindent \textbf{Upper bound:}
We use the same algorithm as in Section~\ref{sec:puredp-tv}. Following the same argument as in Section~\ref{sec:puredp-tv}, the square of expected $\ell_2$ loss of $\hat{p}$ can be upper bounded by
\begin{align}
\Paren{\expectation{\norm{\hat{p}-p}}}^2 &\le \expectation{\norm{h-p}^2} \le \expectation{\norms{p^{\text{erm}}-p}}+\expectation{\norms{h - p^{\text{erm}}}} = O\Paren{\frac1{\ns}+\frac{\ab}{\ns^2\eps^2}}. \nonumber
\end{align}
Since $\Delta_k$ is convex, we have $\norm{\hat{p}-p} \le\norm{h-p}$. Moreover, the following lemma gives another bound for $\norm{\hat{p}-p}$ (See Corollary 2.3 in~\cite{Bassily18}).
\begin{lemma}
Let $L \subset \RR^d$ be a symmetric convex body of $\ab$ vertices $\{a_j\}_{j=1}^{\ab}$, and let $y \in L$ and $\bar{y} = y+z$ for some $z \in \RR^d$. Let $\hat{y} = \arg \min_{w\in L} \norm{w-\bar{y}}^2$. Then we must have 
$$ \norm{y-\hat{y}}^2 \le 4\max_{j\in[k]} \{ \langle z, a_j \rangle\}. $$
\end{lemma} 

From the lemma, we have $\expectation{\norm{\hat{p}-h}^2}\le 4\cdot \expectation{ \max_{j\in[\ab]} \absv{Z_j}}$, where $\forall j \in [\ab]$, $Z_j \sim \text{Lap}(\frac{2}{\ns\eps})$. Note that $\expectation{ \max \absv{Z_j}} = O\Paren{\frac{\log\ab}{\ns\eps}}$ due to the tail bound of Laplace distribution. We have $\Paren{\expectation{\norm{\hat{p}-p}}}^2 = O\Paren{\frac1{\ns}+ \frac{\log\ab}{\ns\eps}}$. Combined with the previous analysis,  $\Paren{\expectation{\norm{\hat{p}-p}}}^2 = O\Paren{\frac1{\ns}+ \min\Paren{\frac{\ab}{\ns^2\eps^2} ,\frac{\log\ab}{\ns\eps}} }$. Therefore $\expectation{\norm{\hat{p}-p}} \le \frac1{10}\alpha$ when $n= O \Paren{\frac{1}{\alpha^2}+ \min \Paren {\frac{\sqrt{k}}{\alpha\eps}, \frac{\log\ab}{\dist^2\eps} }}$.

\medskip

\noindent \textbf{Lower bound:}
We first consider the case when $\dist < \frac1{\sqrt{\ab}}$, where we can derive the lower bound simply by a reduction. By Cauchy-Schwartz inequality, for any estimator $\hat{p}$, $ \expectation{\absvn{\hat{p}-p}} \le \sqrt{\ab} \cdot \expectation{\norm{\hat{p}-p}}$. Therefore $S(\Delta_k,  \ell_2, \alpha, \eps) \ge S(\Delta_k,  d_{TV}, \sqrt{k}\alpha, \eps)$, which gives us $S(\Delta_k,  \ell_2, \alpha, \eps)= \Omega\Paren{\frac{1}{\dist^2}+\frac{\sqrt{\ab}}{\dist\eps}}$.

Now we consider $\dist \ge \frac1{\sqrt{k}}$. Note that it is enough if we prove the lower bound of $\Omega\Paren{\frac{\log \Paren{\alpha^2\ab}}{\dist^2\eps}}$, since $\Omega\Paren{\frac1{\dist^2}}$ 
is the sample complexity of non-private estimation problem for all range of $\dist$. Similarly, we follow Corollary~\ref{coro:fano}, except that we need to construct a different set of distributions. 

Without loss of generality, we assume $\dist<0.1$. Now we use the codebook in Lemma~\ref{lem:GV} to construct our distribution set. {We fix weight $l = \lfloor \frac1{50\dist^2} \rfloor$. Note that for any $x>2$, $\lfloor x \rfloor > \frac{x}{2}$. Then we have  $\frac1{100\dist^2} < \lfloor l \rfloor  \le \frac1{50\dist^2} $ since $\frac1{50\dist^2}>2$. Therefore we get a codebook $\cC$ with $|\cC| \ge (k\alpha^2)^{\frac{1}{200 \alpha^2}}$}. Given $c \in \cC$, we construct the following distribution  $\p_c$ in $\Delta_\ab$:
$$ \p_c (i)  = \frac1{l} c_i. $$

We use $\cV_\ab = \{ \p_c : c \in \cC \}$ to denote the set of all these distributions. It is easy to check that $\forall \p \in \cV_\ab$ is a valid distribution. Moreover, for any pair of distributions $\p, \q \in \cV_\ab$, we have $\norm{\p - \q}> \frac1{2\sqrt{l}} = \Omega\Paren{\dist}$.

For any pair $\p,\q \in \cV_k$, $d_{TV} (\p, \q) \le 1$, which is a naive upper bound for $TV$ distance. Finally by setting $\ell$ in Corollary~\ref{coro:fano} to be $\ell_2$ distance, we have $S(\Delta_{\ab}, \ell_2, \alpha, \eps)  =  \Omega \Paren{\frac{\log \absv{\cC} }{\eps}} = \Omega \Paren{\frac{\log ( k\alpha^2)}{\alpha^2\eps}}$.

\subsection{Product distribution estimation} \label{sec:product}


Recall that $\Delta_{\ab, \dims}$ is the set of all $(\ab,\dims)$-product distributions.~\cite{BunKSW2019} proves an upper bound of $O\Paren{\ab\dims \log \Paren{ \frac{\ab\dims}{\dist}} \Paren{\frac1{\dist^2}+\frac1{\dist\eps}}}$. We prove a sample complexity lower bound for $\eps$-DP $(\ab,\dims)$-product distribution estimation in Theorem~\ref{thm:main_product}, which is optimal up to logarithmic factors.

\begin{theorem}
\label{thm:main_product}
The sample complexity of $\eps$-DP $(\ab,\dims)$-product distribution estimation satisfies
\begin{align}
&S(\Delta_{\ab, \dims}, d_{TV}, \alpha, \eps) = \Omega \Paren{\frac{\ab\dims}{\dist^2}+\frac{\ab\dims}{\dist \eps}}. \nonumber
\end{align}
\end{theorem}

\begin{proof}
{
We start with the construction of the distribution set. First we use the same binary code as in Lemma~\ref{lem:GV} with weight $l = \frac{\ab}{2}$.
Let $h := \absv{\cC}$ denote the size of the codebook. Given $j \in [h]$, we construct the following $\ab$-ary distribution $\p_j$ based on $c_j \in \cC$:}

{$$ \p_{j} (i)  = \frac{1}{\ab} +\frac{\dist}{\ab\sqrt{\dims}} \cdot \mathbb{I}\Paren{ c_{j,i}=1},$$}
{ where $ c_{j,i}$ denotes the $i$-th coordinate of $c_j$.}
%

Now we have designed a set of $\ab$-ary distributions of size $h = \Omega \Paren{2^{\frac{7\ab}{128}}}$. To construct a set of product distributions, {we use the codebook construction in Lemma~\ref{lem:constantGV2} to get an $h$-ary codebook $\cH$ with length $d$ and minimum hamming distance $d/2$. Moreover, $|\cH| \ge (\frac{h}{16})^{\frac{\dims}{2}}$. }

Now we can construct the distribution set of $(\ab,\dims)$-product distributions.
Given $b \in \cH$, define
$$\dP_b = \p_{b_1} \times  \p_{b_2}\times \cdots \times \p_{b_d}.$$

Let $\cV_{\ab,\dims}$ denote the set of distributions induced by $\cH$. We want to prove that $\forall \dP \neq \dQ \in \cV_{\ab,\dims},$
\begin{align}
	\dtv{\dP}{\dQ} \ge C\dist, \label{eqn:dtv_prod}\\
	D_{KL} (\dP,\dQ)  \le 4\dist^2, \label{eqn:dkl_prod}
\end{align}
for some constant $C$. Suppose these two inequalities hold, using~\eqref{eqn:dkl_prod}, by Pinsker's Inequality, we get $d_{TV}(\dP, \dQ) \le \sqrt{2D_{KL} (\dP, \dQ) } \le 2\sqrt2\alpha$. Then using Corollary~\ref{coro:fano}, we can get
\begin{align}
&S(\Delta_{\ab, \dims}, d_{TV}, \alpha, \eps) = \Omega \Paren{\frac{\ab\dims}{\dist^2}+\frac{\ab\dims}{\dist \eps}}. \nonumber
\end{align}
Now it remains to prove~\eqref{eqn:dtv_prod} and~\eqref{eqn:dkl_prod}. For~\eqref{eqn:dkl_prod}, note that for any distribution pair $\dP,\dQ \in \cV_{\ab,\dims}$,
\begin{align}
	D_{KL} (\dP,\dQ) \le \dims \cdot \max_{i, j \in [h]} d_{KL}\Paren{\p_i, \p_j} \le 4\dist^2, \nonumber
\end{align}
\newzs{
where the first inequality comes from the additivity of $KL$ divergence for independent distributions and $\forall i, j \in [h]$, 
\[
	d_{KL}\Paren{\p_i, \p_j} = \sum_{x \in [k]} \p_i(x) \log \frac{\p_i(x)}{\p_j(x)} \le \sum_{x \in [\dims]} \frac{(\p_i(x) - \p_j(x))^2}{\p_j(x)} \le \ab \Paren{\frac{\dist}{\ab\sqrt{\dims}}}^2/ \frac{1}{\ab} = \frac{\alpha^2}{d}.
\]
}
Next we prove~\eqref{eqn:dtv_prod}. For any $b \in \cH$ and $\forall i \in[k]$, define set
\[
	S_i  = \{j \in [\ab] : c_{ b_{i},j} = 1\},
\]
which contains the locations of $+1$'s in the code at the $i$th coordinate of $b$. Based on this, we define a product distribution
\[
	\dP'_{b} = \prod_{i = 1}^{d} \cB(\mu_{i} ),
\]
where $\mu_{i} = \sum_{j \in S_i} \p_{b_{i}} (j)$ and $\cB(t)$ is a Bernoulli distribution with mean $t$. For any $b' \neq b \in \cH$, we define
\[
	\dP'_{b'} = \prod_{i = 1}^{d} \cB(\mu'_{i} ),
\]
 where $\mu'_{i} = \sum_{j \in S_i} \p_{b'_{i}}(j)$. Then we have:
 \[
 	\dtv{\dP'_{b}}{\dP'_{b'}} \le \dtv{\dP_{b}}{\dP_{b'}},
 \]
since $\dP'_{b}$ and $\dP'_{b'}$ can be viewed as a post processing of $\dP_{b}$ and $\dP_{b'}$ by mapping elements in $S_i$ to 1 and others to 0 at the $i$-th coordinate. Moreover, we have $\ham{b}{b'} \ge \frac{d}{2}$, and $\forall i$, if $b_{i} \neq b'_{i}$, we have $d_{H}(c_{b_{i}}, c_{b'_{i}}) > \frac{k}{8}$. By the definition of $\p_i$'s, we have
\[
	\| \mu_1 - \mu_2 \|_2^2 \ge \frac{d}{2} \times \left( \frac{\ab}{8} \times \frac{\dist}{\ab \sqrt{d}} \right)^2 = \frac{\dist^2}{128}.
\] 
By Lemma 6.4 in~\cite{KamathLSU18}, there exists a constant $C$ such that $\dtv{\dP'_{b}}{\dP'_{b'}} \ge C\dist$, proving~\eqref{eqn:dtv_prod}.


\end{proof}

\subsection{Gaussian mixtures 	estimation} \label{sec:gaussian}



Recall $\cG_{\dims} = \{ \cN (\mu, I_d): \norm{\mu}\le R\}$ is the set of $\dims$-dimensional spherical Gaussians with unit variance and bounded mean and $\cG_{\ab,\dims} = \{ \p: \p \text{ is a $\ab$-mixture of } \cG_{\dims} \}$ consists of mixtures of $k$ distributions in $\cG_{\dims}$.~\cite{BunKSW2019} proves an upper bound of $\widetilde{O}\Paren{\frac{\ab\dims}{\dist^2}+\frac{\dims}{\dist\eps}}$ for estimating $k$-mixtures of Gaussians. We provide a sample complexity lower bound for estimating mixtures of Gaussians in Theorem~\ref{thm:main_Gaussian}, which matches the upper bound up to logarithmic factors. 

\begin{theorem}
\label{thm:main_Gaussian}
Given $\ab \le \dims$ and $R\ge \sqrt{64\log\Paren{\frac{8\ab}{\dist}}}$, {or $\ab \ge \dims$ and $R\ge (\ab)^{\frac1\dims} \cdot \sqrt{64\dims \log\Paren{\frac{8\ab}{\dist}}}$,}
\begin{align}
&S(\cG_{\ab, \dims}, d_{TV}, \alpha, \eps) = \Omega \Paren{\frac{\ab\dims}{\dist^2}+\frac{\ab\dims}{\dist \eps}}. \nonumber
\end{align}
\end{theorem}

\begin{proof}
	
{We first consider the case when $\ab \le \dims$ and $R\ge \sqrt{64\log\Paren{\frac{8\ab}{\dist}}}$.} Let $\cC$ denote the codebook in Lemma~\ref{lem:GV} with weight $l=\frac{\dims}{2}$. Then we have $|\cC| \ge 2^{\frac{7\ab}{128}}$. Given $c_i$ in codebook $\cC$, we construct the following $\dims$-dimensional Gaussian distribution $\p_i$, with identity covariance matrix and mean $\mu_i$ satisfying 

\[\mu_{i,j}  = \frac{\dist}{\sqrt{ \dims}} c_{i,j},\]
{ where   $\mu_{i,j}$ denotes the $j$-th coordinate of $\mu_i$.}

{Let $h = \absv{\cC}$. Similar to the product distribution case, using Lemma~\ref{lem:constantGV2}, we can get an $h$-ary codebook $\cH$ with length $d$ and minimum hamming distance $d/2$. Moreover, $|\cH| \ge (\frac{h}{16})^{\frac{\dims}{2}}$. }

$\forall i \in [h]$ and $j \in k$, define $\p_{i}^{(j)} = \cN(\mu_i + \frac{R}{2}e_j , I_d)$, where $e_j$ is the $j$th standard basis vector. It is easy to verify their means satisfy the norm bound. 
For a codeword $b \in \cH$, let
\[
	\mP_b = \frac1k \Paren{ \p_{b_1}^{(1)} +  \p_{b_2}^{(2)}  +\ldots + \p_{b_k}^{(k)} }.
\]

Let $\cV_{\cG} = \{ \mP_b: b \in \cH\}$ be the set of the distributions defined above. Next we prove that $\forall \mP_b \neq \mP_{b'}\in \cV_{\cG}$, 
\begin{align}
	\dtv{\mP_b}{\mP_{b'}} \ge C \alpha, \label{eqn:dtv_gaussian} \\
	\dkl{\mP_b}{\mP_{b'}} \le 4 \alpha^2.\label{eqn:dkl_gaussian}
\end{align}
where $C$ is a constant. If these two inequalities hold, using~\eqref{eqn:dkl_gaussian}, by Pinsker's Inequality, we get $\dtv{\mP_b}{\mP_{b'}} \le \sqrt{2\dkl{\mP_b}{\mP_{b'}}  } \le 2\sqrt2\alpha$. Using Corollary~\ref{coro:fano}, we get
\[
S(\cG_{\ab, \dims}, d_{TV}, \alpha, \eps) =\Omega \Paren{\frac{\ab\dims}{\alpha^2}+\frac{\ab\dims}{\alpha \eps}}.
\]
It remains to prove~\eqref{eqn:dtv_gaussian} and~\eqref{eqn:dkl_gaussian}.

For~\eqref{eqn:dkl_gaussian}, note that for any distribution pair $\mP_b \neq \mP_{b'} \in \cV_{\cG}$,
\begin{align}
\dkl{\mP_b}{\mP_{b'}} &\le \frac1{k} \sum_{t=1}^{\ab} \dkl{\p^{(t)}_{b_t}}{\p^{(t)}_{b'_{t}}}\le \max_{i,j \in [h]} d_{KL}\Paren{\p_i, \p_j} \le 4\dist^2, \nonumber
\end{align}
{where the first inequality comes from the convexity of $KL$ divergence and the last inequality uses the fact that the KL divergence between two Gaussians with identity covariance is at most the $\ell_2^2$ distance between their means.} 

Next we prove~\eqref{eqn:dtv_gaussian}. Let $B_j = B_{j,1} \times \cdots \times B_{j,\dims}$, where

$$B_{j,i}=
\begin{cases}
[ \frac{R}{4}, \frac{3R}{4}], & \text{when $i=j$,}\\
[ -\frac{R}{4}, \frac{R}{4} ], & \text{when $i \neq j$ and $i \le \ab$,}\\
[-\infty, \infty], & \text{when $\ab<i \le \dims$.}
\end{cases}$$

Then by Gaussian tail bound and union bound, for any $\mP \in \cV_{\cG}$, the mass of the $j$-th Gaussian component outside $B_j$ is at most $2ke^{-\frac12 \cdot \Paren{\frac1{4} R}^2}$. And the mass of other Gaussian components inside $B_j$ is at most $e^{-\frac12 \cdot \Paren{\frac1{4} R}^2}$. Hence we have:
\begin{align}
	\dtv{\mP_b}{\mP_{b'}} &= \frac1{2\ab} \int_{z \in \RR^{\dims}} \absv{ \p_{b_1}^{(1)}(z) +\cdots+ \p_{b_{\ab}}^{(\ab)}(z)  - \p_{b'_{1}}^{(1)}(z)  - \cdots -  \p_{b'_{\ab}}^{(\ab)}(z) }  dz \nonumber\\
	&\ge \frac1{2\ab}  \sum_{j=1}^{\ab} \int_{z \in B_j} \absv{ \p_{b_1}^{(1)}(z) +\cdots+ \p_{b_{\ab}}^{(\ab)}(z)  - \p_{b'_{1}}^{(1)}(z)  - \cdots -  \p_{b'_{\ab}}^{(\ab)}(z) }  dz \nonumber\\
	& \ge \frac1{2\ab} \cdot \sum_{j=1}^{\ab} (\int_{z \in B_j}  \absv{  \p_{b_{j}}^{(j)}(z)  - \p_{b'_{j}}^{(j)}(z) } dz -(\ab - 1) \cdot e^{-\frac12 \cdot \Paren{\frac1{4} R}^2}) \nonumber \\
	& \ge \frac1{2\ab} \cdot \sum_{j=1}^{\ab} (\int_{z \in \RR^{\dims}}  \absv{  \p_{b_{j}}^{(j)}(z)  - \p_{b'_{j}}^{(j)}(z) } dz - 3\ab \cdot e^{-\frac12 \cdot \Paren{\frac1{4} R}^2}) \nonumber \\
	& = \frac1{2\ab} \cdot \sum_{j=1}^{\ab} \dtv{\p_{b_{j}}}{\p_{b'_{j}}} - \frac{3\dist^2}{64k}. \nonumber
\end{align}
By Fact 6.6 in~\cite{KamathLSU18}, there exists a constant $C_1$ such that for any pair $i \neq j \in [h]$, 
$$\dtv{\p_i}{\p_j} \ge C_1\dist.$$
Hence we have
\[
	\frac1{2\ab} \cdot \sum_{j=1}^{\ab} \dtv{\p_{b_{j}}}{\p_{b'_{j}}} \ge \frac{C_1\dist}{2\ab} \ham{b}{b'} \ge \frac{C_1\dist}{4},
\]
where the last inequality comes from the property of the codebook. WLOG, we can assume $
\frac{3\dist}{64k} < C_1/8$. Taking $C = \frac{C_1}{8}$ completes the proof of~\eqref{eqn:dtv_gaussian}.

{Now we considers the case when $\ab \ge \dims$ and $R\ge (\ab)^{\frac1\dims} \cdot \sqrt{64\dims \log\Paren{\frac{8\ab}{\dist}}}$. Let $r = \sqrt{16\dims \log\Paren{\frac{8\ab}{\dist}}}$, we note that there exists a packing set $S = \{v_1, v_2, ..., v_k\} \subset \mathbb{R}^{\dims}$ which satisfies $\forall u, v \in S$,
	\[
		\norm{u - v} > r,~~~~ \norm{u}\le R,~~~~, \norm{v}\le R/3, 
	\]
	and $|S| = \ab$ since $R\ge 2 (\ab)^{\frac1\dims}  r$. Consider the set of mixture distributions as following: 
	For a codeword $b \in \cH$, let
	\[
		\mP'_b = \frac1k \Paren{ \p_{b_1}^{(1)'} +  \p_{b_{2}}^{(2)'}  +\ldots + \p_{b_{\ab}}^{(k)'} },
	\]
	where $\forall i \in [k], \p_{b_j}^{(j)'} = \cN(\mu_{b_j} + v_j, I_d)$.
	Let $B'_j$ denote the $\ell_2$ ball centering at the $v_j$ with radius $\frac{r}{2}$. We note that by similar analysis using the tail bound of the Gaussian distribution, the mass of the $j$-th Gaussian component outside $B_j'$ is at most $\frac{\dist^2}{64\ab^2}$. Meanwhile, the mass of other Gaussian components inside $B_j'$ is also at most $\frac{\dist^2}{64\ab^2}$. Hence the remaining analysis follows from the previous case.}
\end{proof}

\section{$(\eps,\delta)$-DP distribution estimation} \label{sec:adp_applications}
{In the previous section we used Theorem~\ref{thm:dp_fano} to obtain sample complexity lower bounds for pure differential privacy. We will now use Theorem~\ref{thm:assouad} to prove sample complexity lower bounds under $(\eps, \delta)$-DP.}

\subsection{$\ab$-ary distribution estimation} \label{sec:adp_kary}


\begin{theorem}
\label{thm:ApproximateDiscreteTotalVariation}
The sample complexity of $(\eps, \delta)$-DP $\ab$-ary distribution estimation under total variation distance is 
\[
S(\Delta_k,  d_{TV}, \alpha, \eps, \delta) =\Omega \Paren{\frac{k}{\alpha^2}+\frac{k}{\alpha(\eps+\delta)}}.
\]
\end{theorem}
{In practice, $\delta$ is chosen to be $\delta = O\Paren{\frac1\ns}$, and the privacy parameter is chosen as a small constant, $\eps = \Theta(1)$. In particular, when $\delta\le \eps$, the theorem above shows
\[
S(\Delta_k,  d_{TV}, \alpha, \eps, \delta) =\Omega \Paren{\frac{k}{\alpha^2}+\frac{k}{\alpha\eps}}.
\]
Since the sample complexity of $\eps$-DP is at most the sample complexity of $(\eps,\delta)$-DP, this shows that the bound above is tight for {$ \delta\le \eps$}. The lower bound part is proved using Theorem~\ref{thm:assouad} in Section~\ref{sec:approximate-tv}.}

\begin{theorem}
\label{thm:ApproximateDiscreteL2}
The sample complexity of $(\eps, \delta)$-DP discrete distribution estimation under $\ell_2$ distance,
\[
 \Omega \Paren {\frac{1}{\alpha^2}+\frac{\sqrt{\ab}}{\alpha (\eps+\delta)}} \le S(\Delta_k, \ell_2, \alpha, \eps, \delta) \le O \Paren{\frac{1}{\alpha^2}+\frac{\sqrt{\ab}}{\alpha\eps}},\ \ \ \ \text{for $\dist < \frac1{\sqrt{\ab}}$},
\] 
\begin{align}
\Omega\Paren{\frac{1}{\dist^2}+\frac{1}{\dist^2(\eps+\delta)}}\le S(\Delta_k, \ell_2, \alpha, \eps, \delta)  \le O\Paren{\frac{1}{\dist^2}+\frac{\log\ab}{\dist^2\eps}}, \ \ \ \text{for $\dist > \frac1{\sqrt{\ab}}$}. \nonumber
\end{align}
\end{theorem}
{When $\delta = O(\eps)$, the bounds are tight when $\dist < 1/\sqrt{\ab}$ and differ by a factor of $\log k$ when $\dist \ge 1/\sqrt{\ab}$. We prove this result in Section~\ref{sec:approximate-l2}.}

\subsubsection{Proof of Theorem~\ref{thm:ApproximateDiscreteTotalVariation}.}
\label{sec:approximate-tv}

{The first term $k/\alpha^2$ is the tight sample complexity without privacy. We prove that $S(\Delta_{\ab}, d_{TV}, \alpha, \eps, \delta) = \Omega \Paren{\frac{\ab}{\alpha(\eps+\delta)}}$.} 

{Suppose $k$ is even and $\dist < 1/10$. Let $ \cE_{k/2} =\{-1,+1\}^{k/2}$, for  $e \in \cE_{k/2}$, we define $\p_e\in\Delta_k$ as follows. 
\begin{align}
\text{For $i=1, \ldots, k/2$}~~~~ \p_e(2i-1) = \frac{1+ 10 e_i \cdot \alpha}{\ab},  ~~ \p_e(2i) = \frac{1- 10e_i  \cdot \dist}{\ab}. \label{eqn:constr-approx-k-ary}
\end{align}}
{To apply Theorem~\ref{thm:assouad}, let $\cV_{\ab/2} = \{ \p_e^n, e \in \cE_{\ab/2}\}$. $\p_e^n$ is the distribution of $n$ i.i.d. samples from distribution $\p_e$, and $\theta(\p_e^n) = \p_e$. For $u, v \in \cE_{\ab/2}$,
\[
\ell( \theta(\p_u^n),  \theta(\p_v^n)) = d_{TV} (p_{u},p_{v}) = \frac{20\dist}{\ab} \cdot \sum_{i=1}^{\frac{k}{2}}\mathbb{I}\Paren{u_i \neq v_i}, 
\]
thus obeying~\eqref{eqn:assouad-loss} with $\tau=10\dist/k$.}
{Recall the mixture distributions $\p_{+i}$ and $\p_{-i}$,} 
\[
\dP_{+i}  =  \frac{2}{|\cE_{\ab/2}|} \sum_{e \in \cE_{\ab/2}: e_i= + 1} \p^{\ns}_e, ~~~ \dP_{-i}  =  \frac{2}{|\cE_{\ab/2}|} \sum_{e \in \cE_{\ab/2}: e_i= -1}\p^{\ns}_e.
\]
{To apply Theorem~\ref{thm:assouad}, we prove the following bound on the Hamming distance between a coupling between $\p_{+i}$ and $\p_{-i}$.
\begin{lemma} \label{lem:coupling_distance}
For any $i$, there is a coupling $(X,Y)$ between $\dP_{+i}$ and $\dP_{-i}$, such that
\[
\expectation{\ham{X}{Y}} \le \frac{20\dist \ns}{\ab}.
\]
\end{lemma}}
\begin{proof}
{By the construction in~\eqref{eqn:constr-approx-k-ary}, note that the distributions $\dP_{+i}$ and $\dP_{-i}$ only have a difference in the number of times $2i-1$ and $2i$ appear. To generate $Y\sim \p_{-i}$ from from $X\sim\p_{+i}$, we scan through $X$ and independently change the coordinates that have the symbol $2i-1$ to the symbol $2i$ with probability $\frac{20\alpha}{1+10\alpha}$. The expected Hamming distance is bounded by $\frac{20\alpha}{1+10\alpha} \cdot \frac{1+10\alpha}{\ab} \cdot \ns = \frac{20\alpha \ns}{\ab}$.}
\end{proof}

Note that $\cV\subset\cP := \{\p^n | \p \in \Delta_k \}$. By Theorem~\ref{thm:assouad}, using the bound on $D$ from Lemma~\ref{lem:coupling_distance}, and $\tau=10\dist/k$,
\[
R(\cP, d_{TV}, \eps, \delta) \ge \frac{5\alpha}{\ab} \cdot \ab \cdot \Paren{0.9 e^{-10 \eps D} - 10 D \delta}\ge  {5\alpha} \cdot \Paren{0.9 e^{-200 n\eps\alpha/k} - 200 \frac{n\eps\alpha\delta}k}.
\]
{To achieve $R(\cP, d_{TV}, \eps, \delta) \le \alpha$, either $n\eps\alpha/k=\Omega(1)$ or ${n\eps\alpha\delta}/k=\Omega(1)$, which implies that $n=\Omega( \frac{k}{\alpha(\eps+\delta)})$}.

\subsubsection{Proof of Theorem~\ref{thm:ApproximateDiscreteL2}} \label{sec:approximate-l2}

We first consider the case where $\dist < \frac1{\sqrt{\ab}}$.  By Cauchy-Schwarz inequality, $S(\Delta_k, \ell_2, \alpha, \eps, \delta) \ge S(\Delta_k,  d_{TV}, \sqrt{k}\alpha, \eps, \delta) $, and therefore $S(\Delta_k, \ell_2, \alpha, \eps, \delta) = \Omega\Paren{\frac{1}{\dist^2}+\frac{\sqrt{\ab}}{\dist (\eps+\delta)}}$ by Theorem~\ref{thm:ApproximateDiscreteTotalVariation}.

For $\alpha \ge \frac{1}{\sqrt{k}}$, we have $l = \lfloor \frac1{16\dist^2}\rfloor \le k$. Therefore, $\Delta_l \subset \Delta_k$ and $\dist < \frac1{\sqrt{l}}$. Hence,
\[
	S(\Delta_k, \ell_2, \alpha, \eps, \delta) \ge S(\Delta_l, \ell_2, \alpha, \eps, \delta) = \Omega\Paren{\frac{1}{\alpha^2} + \frac1{\dist^2(\eps+\delta)}}.
\]

\subsection{Binary product distribution estimation} 
\label{sec:adp_product}

{We now consider estimation of Bernoulli product distributions under total variation distance. A Bernoulli product distribution in $d$ dimensions is a distribution over $\{0,1\}^\dims$ parameterized by $\mu\in [0,1]^\dims$, where the $i$th coordinate is distributed $\ber(\mu_i)$, where $\ber(\cdot)$ is a Bernoulli distribution. Let $\Delta_{2,\dims}$ be the class of Bernoulli product distributions in $\dims$ dimensions.} 

\begin{theorem}
\label{thm:ApproximateProductTotalVariation}
The sample complexity of $(\eps, \delta)$-DP  binary product  distribution estimation satisfies 
\[
S(\Delta_{2,\dims}, d_{TV}, \alpha, \eps, \delta)  =\Omega \Paren{\frac{\dims}{\alpha^2}+\frac{\dims}{\alpha(\eps+\delta)}}.
\]
\end{theorem}

\newzs{Compared to the upper bound of $O\Paren{\dims \log \Paren{\dims/\dist} \Paren{1/\dist^2+1/\dist\eps}}$ in~\cite{BunKSW2019, KamathLSU18}, our bound is tight up to logarithmic factors when $\delta \le \eps$. \cite{KamathLSU18} also presents a lower bound of $\Omega \Paren{\frac{\dims}{\alpha^2}+\frac{\dims}{\alpha(\eps+\delta)}}$ under $(\eps, \delta)$-DP when $\delta = O(1/\ns)$. Although $\delta = O(1/\ns)$ is the more interesting regime in practice, our bound complements the result by stating that the utility will not improve even if $\delta$ can be as large as $\eps$.}
\begin{proof}	
{Since $\Theta(\dims/\eps^2)$ is an established tight bound for non-private estimation, we only prove the second term.}

{We start by constructing a set of Bernoulli product distributions indexed by $\cE_{\dims} = \{ \pm1\}^{\dims}$. For all $e \in \cE_{\dims}$, let $\p_e = \ber(\mu^e_1) \times \ber(\mu^e_2) \times \cdots \times \ber(\mu^e_\dims)$, where 
\[
\mu^e_i = \frac{1+ e_i\cdot 20\dist}{\dims}.
\]}
{Let $\cV = \{ \p_e^{\ns}, e \in \cE_{\dims}\}$, the set of distributions of $n$ i.i.d samples from $\p_e$, and $\theta\Paren{\p_e^{\ns}} = \p_e$. For $u, v\in \cE_{\dims},$, 	$\ell( \theta(\p_u^n),  \theta(\p_v^n)) = d_{TV} (p_{u},p_{v})$. We frist prove that~\eqref{eqn:assouad-loss} holds under total variation distance for an appropriate $\tau$.}


\begin{lemma}
\label{lem:binaryproducttv}
There exists a constant $C_1 > 5 $ such that $ \forall u, v \in \cE_{\dims}$,
\[ 
d_{TV} (\p_{u},\p_{v}) \ge \frac{C_1 \dist}{\dims}\cdot  \sum_{i=1}^{\dims}\mathbb{I}\Paren{u_i \neq v_i}.
\]
\end{lemma}

\begin{proof}
Let $S = \{i \in [\dims] : u_i \neq v_i\}$, and $S^{\prime} = \{i \in S : u_i = 1\} $. WLOG, let $\absv{S^{\prime}} \ge \frac12 \absv{S}$ (or else we can define $S^{\prime} = \{i \in S : u_i = -1\}$). Given a random sample $Z \in \{ \pm1\}^{\dims}$, we define an event $A = \{\forall i \in S^{\prime}, Z_i=0\}$. Now we consider the difference between the following two probabilities, which is a lower bound of the total variation distance between $\p_u$ and $\p_v$.
\begin{align}
d_{TV} (\p_{u},\p_{v}) &\ge \absv{\probofsub{Z \sim \p_{u}} {A} - \probofsub{Z \sim \p_{v}} {A}} \nonumber\\
&= \Paren{1-\frac{1-20\dist}{\dims}}^{\absv{S^{\prime}}} - \Paren{1-\frac{1+20\dist}{\dims}}^{\absv{S^{\prime}}} \nonumber\\
&\ge \frac{40\dist}{\dims} \cdot  \absv{S^{\prime}} \cdot  \Paren{1-\frac{1+20\dist}{\dims}}^{\absv{S^{\prime}}} \nonumber \\
& \ge \frac{40 \dist}{\dims} \cdot \absv{S'} e^{-(1+20\dist)} \ge \frac{C_1\dist}{\dims} \cdot \ham{u}{v}, \nonumber
\end{align}
where in the last two inequalities, we assume $\dims \ge 1000$ and $\dist<0.01$.
\end{proof} 

{Let $D$ be an upper bound on the expected Hamming distance for a coupling between $\dP_{+i}$ and $\dP_{-i}$ over all $i$.  Since $\cV_{\dims} \subset \Delta_{2,\dims}$, applying Theorem~\ref{thm:assouad} with Lemma~\ref{lem:binaryproducttv} we have
\[
R(\cP, d_{TV}, \eps, \delta) \ge \frac{C_1\alpha}{2\dims} \cdot \dims \cdot \Paren{0.9 e^{-10 \eps D} - 10 D \delta} = \frac{C_1\alpha}{2} \cdot \Paren{0.9 e^{-10 \eps D} - 10 D \delta}. 
\]}

{Setting $R(\cP, d_{TV}, \eps, \delta) \le \alpha$, we get $D =\Omega\Paren{\frac1\eps}$ or $D =\Omega\Paren{\frac1\delta}$, or equivalently, $D =\Omega\Paren{\frac1{\eps+\delta}}$. Lemma~\ref{lem:coupling_distance} below shows that we can take $ D =\frac{40\dist\ns}{\dims}$, which proves the result.}
\end{proof}

\begin{lemma} \label{lem:coupling_distance}
There is a coupling between $(X,Y)$ between $\dP_{+i}$ and $\dP_{-i}$, such that
$\expectation{\ham{X}{Y}} \le \frac{40\dist \ns}{\dims}.$
\end{lemma}

\begin{proof}
We generate $Y\sim \dP_{-i}$ from $X\sim\dP_{+i}$ as follows. If the $i$th coordinate of a sample $X$ is $+1$, we independently flip it to $-1$ with probability  $\frac{40\alpha}{1+20\alpha}$ to obtain a sample $Y$. The expected Hamming distance is bounded by $\frac{40\alpha}{1+20\alpha} \cdot \frac{1+20\alpha}{\dims} \cdot \ns = \frac{40\alpha \ns}{\dims}$. 
\end{proof}

\section{Proof of Theorems}
\label{sec:proofs}
\subsection{Proof of DP Le Cam's method (Theorem~\ref{thm:le_cam})}
\label{sec:le}
The proof technique is similar to the proof of coupling lemma in~\cite{AcharyaSZ18}. However, we directly characterize the error probability in Theorem~\ref{thm:le_cam}, which we then use to prove Theorem~\ref{thm:assouad} (DP Assouad's method).

\dplecam*

\begin{proof}
{From~\eqref{eqn:error-prob},} 
	\[
	P_e(\hat \theta, \cP_1, \cP_2) \ge \frac12 \Paren{ \probofsub{ X\sim \p_1 }{ \hat \theta(X) \neq \p_1 } + \probofsub{X\sim \p_2 }{\hat \theta(X) \neq \p_2 }}.
	\]
	
{The first term in Theorem~\ref{thm:le_cam} follows from the classic Le Cam's method (Lemma 1 in~\cite{Yu97}). For the second term, let $(X, Y)$ be distributed according to a coupling of $\p_1$ and $\p_2$ with $\expectation{\ham{X}{Y}} \le D$. 
	By Markov's inequality, 
	$\probof{\ham{X}{Y}>10D}<0.1$.
	Let $x$ and $y$ be the realization of $X$ and $Y$.  $W := \{ (x,y) \in \cX^n \times \cX^n| \ham{x}{y} \le 10D \}$ be the set of pairs of realizations with Hamming distance at most $10D$.} Then we have 
	\begin{align}
	\probof{\alg\Paren{X}=\p_2}  = & \sum_{x,y} \probof{X = x, Y = y} \cdot \probof{\alg\Paren{x}=\p_2} \nonumber\\
	\ge &  \sum_{(x,y)\in W} \probof{X = x, Y = y} \cdot \probof{\alg\Paren{x}=\p_2}.
	\end{align}
Let  $\beta_1=\probof{\alg\Paren{X}=\p_2} $, so we have 
	$$ \sum_{(x,y)\in W} \probof{X = x, Y = y}  \cdot\probof{\alg\Paren{x}=\p_2} \le \beta_1$$
	
	Next, we need the following group property of differential privacy.
	
	\begin{lemma} \label{lem:group}
		Let $\alg$ be a $(\eps,\delta)$-DP algorithm, then for sequences $x$, and $y$ with $\ham{x}{y} \le t$, and $\forall S$, 
		$		\probof{\hat \theta(x)\in S}\le e^{t\eps}\cdot\probof{\hat \theta(y)\in S}+ \delta t e^{\eps (t-1)}$. 
	\end{lemma}
	
	\noindent By Lemma~\ref{lem:group}, and $\probof{\ham{X}{Y}>10D}<0.1$, let $\probof{\alg\Paren{Y}=\p_2}=1- \beta_2$,
	\begin{align}
	1- \beta_2=&  
	\sum_{(x,y)\in W} \probof{x,y}  \cdot \probof{\alg\Paren{y}=\p_2}+ 
	\sum_{(x,y)\notin W} \probof{x,y} \cdot \probof{\alg\Paren{y}=\p_2} \nonumber \\
	\leq& \sum_{(x,y)\in W} \probof{x,y}  \cdot \Paren{ e^{\eps \cdot 10D} \probof{\alg\Paren{x}=\p_2} + 10 D\delta \cdot e^{\eps \cdot 10(D-1)}} +0.1	 \nonumber \\
	\leq& \beta_1 \cdot e^{\eps\cdot 10 D} +  10 D\delta \cdot e^{\eps \cdot 10D} + 0.1. \nonumber
	\end{align}
	Similarly, we get
	\begin{align}
	1- \beta_1 \le  \beta_2 \cdot e^{\eps\cdot 10 D} +  10 D\delta \cdot e^{\eps \cdot 10D} + 0.1. \nonumber
	\end{align}
	
	Adding the two inequalities and rearranging terms,
	\[
	\beta_1 + \beta_2 \ge \frac{1.8 - 20D\delta e^{\eps \cdot 10D} }{1 + e^{\eps\cdot 10 D}} \ge 0.9 e^{-10 \eps D} - 10 D \delta.
	\]
\end{proof}

\subsection{Proof of private Fano's inequality (Theorem~\ref{thm:dp_fano})}
\label{sec:fano_pro}
In this section, we prove $\eps$-DP Fano's inequality (Theorem~\ref{thm:dp_fano}), restated below. 

\dpfano*

The proof is based on the observation that if you can change a sample from $\p_i$ to $\p_j$ by changing $D$ coordinates in expectation, then an algorithm that  algorithm that correctly outputs a sample as from $\p_i$ has to output $\p_j$ with probability roughly $e^{-\eps D}$. With a total of $M$ distributions in total, we show that the error probability is large as long as $\frac{M}{e^{\eps D}}$ is large.

\begin{proof}
{The first term in~\eqref{eqn:fano_result} follows from the non-private Fano's inequality (Lemma~3 in~\cite{Yu97}).}
{For an observation $X\in\cX^\ns$,} 
\[
	\hat{\p}(X) := \arg \min_{\p \in \cV} \ell \Paren{\theta(\p),  \hat \theta(X)},
\]
{be the distribution in $\cP$ closest in parameters to an $\eps$-DP estimate $\hat \theta(X)$. Therefore, $\hat{\p}(X)$ is also $\eps$-DP. By the triangle inequality,
	\[
	\ell\Paren{\theta(\hat \p), \theta(\p)} \le \ell \Paren{\theta(\hat \p), \hat \theta(X)}  + \ell \Paren{ \theta(\p), \hat \theta(X)} \le 2\ell \Paren{\theta(\p), \hat \theta(X)}.
	\]}
	Hence,
	\begin{align} 
	\max_{\p \in \cP} \EE_{X \sim \p} \left[\ell(\hat \theta (X),\theta (\p)) \right]  
	\ge \max_{\p \in \cV} \EE_{X \sim \p} \left[\ell(\hat \theta (X),\theta (\p)) \right]
	&\ge \frac{1}{2}\max_{\p \in \cV} \EE_{X \sim \p} \left[\ell(\theta(\hat \p) ,\theta (\p)) \right]  \nonumber \\
	& \ge \max_{\p \in \cV} \frac{\alpha}{2}\probofsub{X \sim \p}{\hat{\p} (X) \neq \p} \nonumber \\
	&\ge \frac{\alpha}{2M}\sum_{\p \in \cV} \probofsub{X \sim \p}{\hat{\p} (X) \neq \p}.
	\label{eqn:reduction}
	\end{align}
	Let $\beta_i = \probofsub{X \sim \p_i}{\hat{\p} (X) \neq \p_i}$ be the probability that $\hat p(X)\ne\p_i$ when the underlying distribution generating $X$ is $\p_i$. For $\p_i, \p_j \in \cV$, let $(X,Y)$ be the coupling in condition $(c)$. By Markov's inequality $\probof{\ham{X}{Y}>10D}<1/{10}.$

{Similar to the proof of Theorem~\ref{thm:le_cam} in the previous section, let $W := \{ (x,y) | \ham{x}{y} \le 10D \}$ and $\probof{x,y}  = \probof{X = x, Y = y}$ . Then
	\begin{align}
	1- \beta_j = \probof{\hat{\p}\Paren{Y} =\p_j}  \leq  \sum_{(x, y)\in W} \probof{x,y}  \cdot \probof{\hat{\p} \Paren{y}  = \p_j} + \sum_{(x, y)\notin W} \probof{x,y}  \cdot 1. \nonumber
	\end{align}}
Therefore, 
	\[ \sum_{(x, y)\in W} \probof{x,y}  \cdot \probof{\hat{\p} \Paren{y}  = \p_j} \ge 0.9 - \beta_j.\] Then, we have
	\begin{align}
	\probof{\hat{\p} \Paren{X}  = \p_j} & \ge  \sum_{(x,y)\in W} \probof{x, y} \cdot \probof{\hat{\p} \Paren{x}  = \p_j}\nonumber \\ 
	&\ge \sum_{(x,y)\in W} \probof{x, y} e^{-10 \eps D} \probof{\hat{\p} \Paren{y}  = \p_j} \label{eqn:step-dp} \\
	& \ge ( 0.9 - \beta_j)  e^{-10 \eps D},\nonumber
	\end{align}
	where~\eqref{eqn:step-dp} uses that $\hat p$ is $\eps$-DP and $\ham{x}{y} \le 10D$. Similarly, for all $j' \neq i$, 
	\[
	\probof{\hat{\p} \Paren{X}  = \p_{j'}} \ge ( 0.9 - \beta_{j'})  e^{-10 \eps D}.
	\]
	Summing over $j'\ne i$, we obtain
	\begin{align}
	\beta_i &= \sum_{j' \neq i} \probof{\hat{\p} \Paren{X} = \p_{j'}}  \ge \Paren{ 0.9(M-1) - \sum_{j' \neq i}\beta_{j'} } e^{-10 \eps D}.\nonumber
	\end{align}
	Summing over $i \in [M]$,
	\[
	\sum_{i \in [M]} \beta_i \ge  \Paren{ 0.9 M (M-1) - (M - 1)\sum_{i \in [M]} \beta_{i} } e^{-10 \eps D}.
	\]
	Rearranging the terms
	\[
	\sum_{i \in [M]} \beta_i \ge \frac{0.9 M (M - 1)}{M - 1 + e^{10 \eps D}} \ge 0.8M \min\left\{1,  \frac{M}{e^{10 \eps D}}\right\}.
	\]
	Combining this with~\eqref{eqn:reduction} completes the proof.
\end{proof}

\subsection{Proof of Corollary~\ref{coro:fano}} \label{sec:coro_fano_proof}
%
\begin{proof}
Recall that $\cQ^{\ns} := \{ \q^\ns | \q \in \cQ \}$ is the set of induced distributions over $\cX^\ns$ and $\q^\ns \in \cQ^\ns, \theta(\q^\ns) = \q$. Then,
	$\forall i \neq j \in [M]$,
$\ell \Paren{\theta(\q_i^n),\theta(\q_j^n)} \ge 3 \tau$, and $D_{KL} \Paren{\q_i^n,\q_j^n} = n D_{KL} \Paren{\q_i,\q_j}  \le n \beta.$

{The following lemma is a corollary of maximal coupling~\cite{Hollander12}, which states that for two distributions there is a coupling of their $n$ fold product distributions with an expected Hamming distance $n$ times their total variation distance.}
	\begin{lemma}
		\label{lem:couplingdistance}
		Given distributions $\q_1,\q_2$ over $\cX$, there exists a coupling $(X,Y)$ between $\q_1^n$ and $\q_2^n$ such that  
		\[\expectation{\ham{X}{Y}} = \ns \cdot d_{TV} \Paren{\q_1,\q_2},
		\]
		where $X \sim \q_1^n$ and $Y \sim \q_2^n.$
	\end{lemma}
\noindent By Lemma~\ref{lem:couplingdistance}, $\forall i,j \in [M]$, there exists a coupling $(X, Y)$ between $\q_i^n$ and $\q_j^n$ such that $\expectation{\ham{X}{Y}} \le n\gamma. $
	Now by Theorem~\ref{thm:dp_fano},
	\begin{align}
	R(\cQ^*, \ell, \eps) \ge \max \Bigg\{\frac{3\tau}{2} \left(1 - \frac{n \beta + \log 2}{\log M}\right), 1.2\tau \min\left\{1, \frac{M}{e^{10\eps n \gamma}}\right\} \Bigg\}.
	\end{align}
Therefore, for $R(\cQ^\ns, \ell, \eps) \le \tau$, 
	\[
	S(\cQ, \ell, \tau, \eps) = \Omega\Paren{\frac{\log M}{\beta}+ \frac{\log M}{\gamma \eps}}.
	\]
\end{proof}

\subsection{{Proof of private Assouad's method (Theorem~\ref{thm:assouad})}}
\label{sec:private_assouad}
\newzs{We restate the theorem below and the notions are the same as defined in Section~\ref{sec:dp-assouad}.}
\dpassouad*
\begin{proof}
{The first part is from the non-private Assouad's lemma, which we include here for completeness. Let $\p\in\cV \subset \cP$ and $X\sim\p$. For an estimator $\hat \theta(X)$, consider an estimator $\he(X) = \arg \min_{e \in \cE_{\ab}} \ell \Paren{\hat \theta(X), \theta(\p_e)}$.} Then, by the triangle inequality,
		\[
	\ell \Paren{ \theta(\p_{\he}), \theta(\p) } \le \ell \Paren{ \hat \theta, \theta(\p_{\he})}  + \ell \Paren{ \hat \theta, \theta(\p) } \le 2 \ell \Paren{ \hat \theta, \theta(\p) }.
	\]
	Hence,
	\begin{equation} \label{eqn:reduce_index}
	R(\cV, \ell, \eps, \delta) = \min_{\hat \theta \text{ is }(\eps,\delta)-DP}\ \max_{\p \in \cV}\ \EE_{X \sim \p}\left[\ell(\hat \theta (X),\theta(p))\right] \ge \frac{1}{2} \min_{\he \text{ is }(\eps,\delta)-DP}\ \max_{\p \in \cV}\ \EE_{X \sim \p}\left[\ell(\theta(\p_{\he(X)}),\theta(p))\right]
	\end{equation}
	
\noindent For any $(\eps, \delta)$-DP index estimator $\he$, and by~\eqref{eqn:assouad-loss},
	\begin{align}
	\max_{\p \in \cV }\expectationsub{X \sim \p} {\ell(\theta(\p_{\he}), \theta(\p))} \geq \frac1{|\cE_{\ab}|} \sum_{ e \in \cE_{k}} \expectationsub{X \sim \p_{e}} {\ell(\theta(\p_{\he}),\theta(\p_{e}) )} 
	\geq  \frac{2\tau}{|\cE_{\ab}|}  \sum_{i=1}^{k}  \sum_{ e \in \cE_{k}} \probof{\he_i \neq {e}_i| E = e }. \nonumber 
	\end{align}
	%
	
\noindent For each $i$, we divide $ \cE_{k} =  \{\pm1\}^{k}$ into two sets according to the value of $i$-th position,
	\begin{align}
	\max_{\p \in \cV }\expectationsub{X \sim \p} {\ell(\theta(\p_{\he}), \theta(\p))} & \geq  \frac{2\tau}{|\cE_{\ab} |} \sum_{i=1}^{k}  \Brack {\sum_{e:e_i=1} \probof{ \he_i\neq 1| E=e }+\sum_{e:e_i=-1} \probof{\he_i \neq -1  | E = e} }\nonumber\\
	&= \tau \cdot \sum_{i=1}^{\ab}  \Paren{ \probofsub{ X\sim \p_{+i} }{ \he_i\neq 1 } + \probofsub{X\sim \p_{-i} }{\he_i\neq -1 }}  \nonumber \\
	&\ge \tau \cdot \sum_{i=1}^{\ab} \min_{ \phi_i: \phi_i \text{ is DP}}  \Paren{ \probofsub{ X\sim \p_{+i} }{ \phi_i(X) \neq 1 } + \probofsub{X\sim \p_{-i} }{\phi_i(X) \neq -1 }}.  \nonumber
	\end{align}
	Combining with~\eqref{eqn:reduce_index}, we have
	\[
	R(\cP, \ell, \eps, \delta)  \ge R(\cV, \ell, \eps, \delta) \ge  \frac{\tau}{2} \cdot \sum_{i=1}^{\ab} \min_{ \phi_i: \phi_i \text{ is DP}}  \Paren{ \probofsub{ X\sim \p_{+i} }{ \phi_i(X) \neq 1 } + \probofsub{X\sim \p_{-i} }{\phi_i(X) \neq -1 }},
	\]
	proving the first part.

	For the second part. Note that for each $i \in [\ab]$, the summand above is the error probability of hypothesis testing between the mixture distributions $\p_{+i}$ and $\p_{-i}$. Hence, using Theorem~\ref{thm:le_cam}, 
	\[
	R(\cP, \ell, \eps, \delta) \ge \frac{\ab \tau}{2} \cdot \Paren{0.9 e^{-10 \eps D} - 10 D \delta}.
	\]
\end{proof}

\section{Proofs of existence of codes (Lemma~\ref{lem:GV} and Lemma~\ref{lem:constantGV2})}
\label{sec:codes}

\begin{proof}[Proof of Lemma~\ref{lem:GV}]
{This proof is a standard argument for Gilbert-Varshamov bound applied to constant weight codes. We use the following version (Theorem 7 in~\cite{GrahamS80}).}
\begin{lemma}
\label{lem:constantGV}
{There exists a length-$k$ constant weight binary code $\cC$ with weight $l$ and minimum Hamming distance $2\delta$, with $$\absv{\cC} \ge \frac{\binom{\ab}{l}}{\sum_{i=0}^{\delta} \binom{l}{i} \binom{\ab-l}{i}}.$$}
\end{lemma}
Applying this Lemma with $2\delta=\frac{l}{4}$, we have
\begin{align}
\absv{\cC} &\ge \frac{\binom{\ab}{l}} {\sum_{j=0}^{l/8} \binom{l}{j} \cdot \binom{\ab-l}{j}}\nonumber  \ge \frac{\binom{\ab}{l}}{\frac{l}{8} \cdot \binom{l}{\frac{l}{8}} \cdot \binom{\ab}{\frac{l}{8}}}
  = \frac1 {\frac{l}{8} \cdot \binom{l}{\frac{l}{8}} } \cdot \prod_{i = 0}^{\frac{7l}{8}- 1} \frac{k - \frac{l}{8}-i}{l - i} \\
& \ge \frac{2\sqrt{7} \pi}{e} \cdot (0.59)^{\frac{7l}{8}} \cdot \Paren{ \frac{\ab-\frac{l}{8}}{l}}^{\frac{7l}{8}} \label{eqn:gv-step}\\
 &\ge \Paren{\frac{\ab}{2^{7/8}l}}^{\frac{7l}{8}}, \nonumber 
\end{align}
{In~\eqref{eqn:gv-step}, we note that $\frac{k - \frac{l}{8}-i}{l - i}$ is monotonically increasing as $i$ increases.} And the first part is obtained by the Stirling's approximation $\sqrt{2\pi}\cdot l^{l+\frac12} \cdot e^{-l} \le l! \le e \cdot l^{l+\frac12} \cdot e^{-l}$ and the fact that $1.1^l \ge \sqrt{l}$ when $l \ge 20$. The last inequality comes from $l \le k/2$ and $15/16 \times 0.59 > 1/2^{7/8}$.
\end{proof}

\begin{proof}[Proof of Lemma~\ref{lem:constantGV2}]
By the Gilbert-Varshamov bound (Lemma~\ref{lem:constantGV}),
\[
	\absv{\cH} \ge \frac{h^\dims}{\sum_{j=0}^{\frac{d}{2}-1} \binom{\dims}{j}(h-1)^j} \ge \frac{h^\dims}{\frac{\dims}{2} \cdot \binom{\dims}{\frac{\dims}{2}} \cdot h^{\frac{\dims}{2}}} \ge
	\frac{h^{\frac{\dims}{2}}} { \dims \cdot 2^{\dims} }  \ge \left(\frac{h}{16}\right)^{\frac{\dims}{2}}.
\]
\end{proof}

\section*{Acknowledgements}
The authors thank Gautam Kamath and Ananda Theertha Suresh for their thoughtful comments and insights that helped improve the paper. 

\bibliographystyle{alpha} 
\bibliography{abr,masterref}

\end{document}